\documentclass{clear2025}

\usepackage{times}
\usepackage{booktabs}
\usepackage{tikz}
\usepackage{xcolor}
\usepackage{float}
\usepackage{xstring}
\usepackage{graphicx}
\usepackage{adjustbox}
\usetikzlibrary{arrows,positioning}
\usepackage{caption}
\usepackage{framed}

\usepackage[most]{tcolorbox}

\title[Proxy-Guided Measurement Calibration]{Proxy-Guided Measurement Calibration}

\clearauthor{
 \Name{Saketh Vishnubhatla} \Email{svishnu6@asu.edu}\\
 \addr Arizona State University, Tempe, AZ
 \AND
 \Name{Shu Wan} \Email{swan16@asu.edu}\\
 \addr Arizona State University, Tempe, AZ
 \AND
 \Name{Andre Harrison} \Email{andre.v.harrison2.civ@army.mil}\\
 \addr DEVCOM Army Research Lab, Adelphi, MD
 \AND
\Name{Adrienne Raglin} \Email{adrienne.raglin2.civ@army.mil}\\
 \addr DEVCOM Army Research Lab, Adelphi, MD
 \AND
 \Name{Huan Liu} \Email{huanliu@asu.edu}\\
 \addr Arizona State University, Tempe, AZ
}

\jmlrproceedings{}{}    
\jmlrvolume{}           
\jmlryear{}             
\jmlrworkshop{}         

\begin{document}

\maketitle
\thispagestyle{plain}

\begin{abstract}
Aggregate outcome variables collected through surveys and administrative records are often subject to systematic measurement error. For instance, in disaster loss databases, county-level losses reported may differ from the true damages due to variations in on-the-ground data collection capacity, reporting practices, and event characteristics. Such miscalibration complicates downstream analysis and decision-making. We study the problem of outcome miscalibration and propose a framework guided by proxy variables for estimating and correcting the systematic errors. We model the data-generating process using a causal graph that separates latent content variables driving the true outcome from the latent bias variables that induce systematic errors. The key insight is that proxy variables that depend on the true outcome but are independent of the bias mechanism provide identifying information for quantifying the bias. Leveraging this structure, we introduce a two-stage approach that utilizes variational autoencoders to disentangle content and bias latents, enabling us to estimate the effect of bias on the outcome of interest. We analyze the assumptions underlying our approach and evaluate it on synthetic data, semi-synthetic datasets derived from randomized trials, and a real-world case study of disaster loss reporting. Our code~\footnote{https://github.com/sak-18/proxy-guided-calibration} will be publicly available.

\end{abstract}



\section{Introduction}

Measurements obtained in empirical studies or administrative processes often deviate from the true outcome of interest due to systematic measurement errors \citep{imai2010causal}. These errors arise due to various factors, such as group-specific practices in data collection or a lack of necessary infrastructure \citep{van2019causes}. In such settings, quantifying the effect of outcome miscalibration is crucial for adjusting the observed measurements. Common approaches in the presence of measurement bias include running sensitivity tests to assess the validity of the hypothesis or calibration strategies that rely on validation data for which the true outcome is observed \citep{vanderweele2019simple,guerdan2023counterfactual}. Sensitivity tests can help assess the robustness of the findings, although they do not correct the measurement directly. Having access to true outcomes for a validation subset is a valuable avenue for correcting measurement errors. However, in many real-world scenarios, this may be a strong assumption and often infeasible. In such settings, modeling measurement errors requires additional structure. 

A useful signal comes from \emph{proxy variables}: measurements that are correlated with the underlying outcome but not affected by the underlying systematic error mechanism itself. For example, disaster losses are collected on the ground through windshield surveys and on-site inspections, where losses are reported in terms of human lives impacted and property or crop losses in dollars. In this context, using sensor-based measurements, which are entirely independent of the loss data collection mechanism, can function as proxies. When such proxy information is available, it suggests a principled way to separate the latent ``content’’ driving the true outcome from the latent ``bias’’ driving the measurement errors. 

Building on this causal perspective, we introduce a proxy-guided measurement calibration framework that utilizes variational autoencoders to disentangle unbiased content latents from bias-specific latents. The learned representation allows us to estimate the magnitude of the bias effect. We specify the assumptions required for identification and, through our empirical analysis, demonstrate that our method reliably recovers the underlying latent factors and accurately estimates the bias effect across a range of settings.
\section{Related Work}

\subsection{Systematic Bias in Outcome Measurement}
Several empirical studies demonstrate that measurement processes in crowdsourced urban data introduce systematic biases. \citet{agostini2024bayesian} frame the study in the context of flood reporting across New York City census tracts. They demonstrate that resident-reported incidents exhibit systematic spatial variation in reporting intensity based on the socio-economic and demographic covariates. A Bayesian approach using an Ising model is used to correct for the underreporting, and to predict future event occurrence accurately.  Complementing this perspective, \citet{liu2024quantifying} provide direct empirical evidence of spatial disparities in resident crowdsourcing across New York and Chicago. They propose a method to identify reporting rates solely from the logs of requests filed, without access to ground-truth event incident rates. 

Reporting bias also arises in administrative records of public services, where observed usage reflects both demand and access.
\citet{liu2024identifying} study public library systems and demonstrate that circulation and usage statistics can obscure underlying disparities in service access. Similar concerns have long been raised in the context of disaster loss databases. \citet{gall2009losses} document six recurring fallacies in natural hazards loss databases, demonstrating that reported losses depend on thresholds for inclusion, spatial and temporal aggregation choices, and institutional accounting practices, leading to systematic bias across regions and events. Together, these studies underscore that observed outcomes in many applications are shaped by the reporting mechanisms, motivating approaches that explicitly account for systematic reporting biases when drawing conclusions from real-world data.

A related line of work addresses systematic errors from a machine learning perspective.  In machine learning, noisy or biased labels are often modeled via a label corruption process, such as a noise transition matrix, and learning is adjusted accordingly; recent results show that such noise models can be estimated even without idealized anchor points by leveraging high-confidence examples \citep{xia2019anchor}. 
In many applied settings, however, label noise reflects a deeper form of outcome mismeasurement, where observed labels serve as proxies for underlying quantities whose relationship to the true outcome varies systematically across contexts or subpopulations \citep{9729424}. This phenomenon has been formalized as causal label bias, highlighting how reliance on proxy outcomes can mask disparities and lead to misleading conclusions about performance or fairness when evaluated solely on observed labels \citep{10.1145/3630106.3658972}.  From a causal perspective, outcome mismeasurement can further invalidate standard counterfactual identification; recent work studies counterfactual prediction under outcome measurement error \citep{guerdan2023counterfactual}, derives partial identification bounds for causal effects under differential misreporting \citep{huang2022conditional}, and develops sensitivity analysis frameworks to assess how severe outcome mismeasurement would need to be to explain away observed effects \citep{vanderweele2019simple}. 

Together, these strands emphasize that outcome miscalibration is often systematic rather than random, motivating approaches that explicitly reason about and correct measurement bias rather than treating noisy labels as incidental corruption.

\subsection{Latent Variable Models and Identifiability}

Latent variable modeling has become a central tool for causal inference when key variables are unobserved. A prominent example is addressing unobserved confounding in observational studies: \citet{louizos2017causal} introduce CEVAE, one of the first works to integrate deep latent variable models with causal effect estimation by using proxy variables to infer latent confounders. Related work applies latent variable models to settings where the outcome of interest is observed only through multiple noisy measurements, showing that causal effects on the underlying outcome can be recovered via an optimally weighted combination of proxies \citep{fu2025causal}.

A central challenge for such models is identifiability. In deep generative models, this issue is closely linked to posterior collapse, where latent representations become uninformative. \citet{wang2021posterior} show that posterior collapse arises precisely when latent variables are non-identifiable, and propose identifiable VAE constructions that enforce recoverability through structural constraints. Complementary theoretical results establish identifiability for broad classes of VAE-based models under suitable assumptions, even without auxiliary supervision \citep{kivva2022identifiability}, connecting these developments to classical results in statistics and causal discovery where additional assumptions, such as non-Gaussianity, enable identification of latent factors and causal structure \citep{hyvarinen2024identifiability}. Together, this work informs how latent variable models can be designed to recover meaningful latent structure for causal analysis.

\section{Preliminaries}

\subsection{Notation}

Throughout the paper, lowercase letters (e.g., $k$) denote scalars, and lowercase
boldface letters (e.g., $\mathbf{v}$) denote vectors in $\mathbb{R}^d$. Uppercase
boldface letters (e.g., $\mathbf{W}$) denote matrices, and calligraphic letters
(e.g., $\mathcal{S}$) denote sets. Plain uppercase letters such as $E$ denote
random variables or nodes in a causal graph. Lowercase variables denote
realizations of the corresponding random variables. We write $P(\cdot)$ for probability
mass functions, $p(\cdot)$ for densities, $\mathbb{E}[\cdot]$ for expectations,
and $do(\cdot)$ for Pearl's intervention operator.

In our model, $E$ denotes observed environment covariates. The latent content
vector is written $\mathbf{z} \in \mathbb{R}^{d_z}$, and the latent bias vector
is written $\mathbf{a} \in \mathbb{R}^{d_a}$. The unobserved true outcome is
$Y_{\mathrm{true}}$, and the biased observed measurement is $Y_{\mathrm{obs}}$.
We observe a vector of $m$ proxy measurements,
\[
\mathbf{y}_{\mathrm{proxy}}
=
\big(y_{\mathrm{proxy}}^{(1)}, \ldots, y_{\mathrm{proxy}}^{(m)}\big)
\in \mathbb{R}^m,
\]
which depend on the latent content $\mathbf{z}$ but not influenced by the
bias mechanism $\mathbf{a}$. When convenient, we write $Z$ and
$A$ for the corresponding random variables and $\mathbf{z}$ and
$\mathbf{a}$ for their unobserved realized values and $\hat{z},\hat{a}$ for point estimates.

\subsection{Proxy-Guided Measurement Calibration}

We define \emph{proxy-guided measurement calibration} as the task of recovering the unobserved true outcome $Y_{\mathrm{true}}$ from a biased observed measurement $Y_{\mathrm{obs}}$ by leveraging auxiliary proxy variables that are not causally influenced by the bias mechanism. In many real-world settings, systematic reporting errors cause $Y_{\mathrm{obs}}$ to deviate from $Y_{\mathrm{true}}$. Such errors arise from latent bias factors $\mathbf{a}$ that may vary across
environments indexed by $E$, while the underlying physical or domain-specific
signal is governed by latent content factors $\mathbf{z}$.

Proxy-guided calibration formalizes this setup through a causal model in
which the proxies serve as ``clean'' measurements of $\mathbf{z}$, enabling the
disentanglement of the content factors $\mathbf{z}$ from the bias factors
$\mathbf{a}$. This separation allows us to conceptually define the de-biased counterfactual
outcome $Y_{\mathrm{obs}}(0),$ which corresponds to the true outcome $Y_{true}$ in the model. Recovering $Y_{\mathrm{obs}}(0)$ from observational data would require assumptions we detail in the following sections.

\section{Proxy-Guidance for Measurement Calibration}

Our framework utilizes proxy measurements to separate true latent factors from the bias mechanism. We present the conditions under which we can identify the unobserved true outcome, describe how they are recovered via a two-stage VAE co-training setup, and demonstrate how the learned representations help quantify the effect of the bias.

\subsection{Generative Model}

We assume each observation is generated according to the causal graph in
Figure~\ref{fig:pgmc_dag_combined}. We observe an environment variable
$E=(E_1,\dots,E_{d_e})$, which determines the latent factors. The latent content and bias variables are denoted by $Z=(Z_1,\dots, Z_{d_z})$ and $A$. The unobserved true outcome is $Y_{\mathrm{true}}$, the biased observed outcome is $Y_{\mathrm{obs}}$, and the proxy measurements are collected in the vector-valued random variable $Y_{\mathrm{proxy}}$. The generative process is defined as below:
\[
E \sim p(E),
\qquad
Z \sim p(Z \mid E),
\qquad
A \sim p(A \mid E),
\]
\[
Y_{\mathrm{true}} \sim p(Y_{\mathrm{true}} \mid Z),
\qquad
Y_{\mathrm{proxy}} \sim p(Y_{\mathrm{proxy}} \mid Z),
\]
\[
Y_{\mathrm{obs}} \sim p(Y_{\mathrm{obs}} \mid Z, A).
\]
The joint distribution factorizes as:
\[
p(E,Z,A,Y_{\mathrm{true}},Y_{\mathrm{proxy}},Y_{\mathrm{obs}})
=
p(E)\,
p(Z\mid E)\,
p(A\mid E)\,
p(Y_{\mathrm{true}}\mid Z)\,
p(Y_{\mathrm{proxy}}\mid Z)\,
p(Y_{\mathrm{obs}}\mid Z,A),
\]
which encodes that the proxy measurements depend only on the latent content
variable $Z$, while the observed outcome depends on both content and bias.

\subsection{Identifiability}

We study the problem of identifying the effect of reporting bias on the observed
outcome. For the purposes of identification and estimation in this work, we focus on the
scalar binary case $d_a = 1$, where $A \in \{0,1\}$ indicates the absence or
presence of reporting bias. Our goal is to recover
the expected observed outcome under removal of reporting bias. For any observation with environment value $e$ and latent content value $z$, our causal estimand of interest is:
\[
\mu(e,z)
:= 
\mathbb{E}\!\left[\,Y_{\mathrm{obs}} \mid do(A=0),\,E=e,\,Z=z\,\right],
\]
which represents the bias-free conditional mean of the observed outcome. Because our representation learning procedure learns the latent variables $(Z,A)$ up to trivial transformations, we treat $(E,Z,A)$ as observed for the identification argument. The DAG in Figure~\ref{fig:pgmc_dag_combined} implies two structural conditions: (i) all parents of $A$ are contained in $(E,Z)$, and (ii) $E$ has no parents. These conditions ensure that $(E,Z)$ is a valid adjustment set for the causal effect of $A$ on $Y_{\mathrm{obs}}$.

\begin{proposition}[Identification]
For any $(e,z)$ in the support of $(E,Z)$,
\[
\mu(e,z)
:=
\mathbb{E}\!\left[\,Y_{\mathrm{obs}} \mid do(A=0),\,E=e,\,Z=z\,\right]
=
\mathbb{E}\!\left[\,Y_{\mathrm{obs}} \mid A=0,\,E=e,\,Z=z\,\right].
\]
\end{proposition}

\begin{proof}[Sketch]
By definition,
\[
\mu(e,z)
=
\int y \;
p\!\left(y \mid do(A=0),\,E=e,\,Z=z\right)\,dy.
\]
In the DAG of Figure~\ref{fig:pgmc_dag_combined}, all parents of $A$ are contained in
$(E,Z)$ and $E$ has no parents. Thus $(E,Z)$ is an admissible adjustment set
for the edge $A \to Y_{\mathrm{obs}}$: conditioning on $(E,Z)$ blocks all
backdoor paths from $A$ to $Y_{\mathrm{obs}}$ and contains no descendants of
$A$. By the backdoor criterion (a special case of Rule~2 of do-calculus),
this implies
\[
p\!\left(y \mid do(A=0),\,E=e,\,Z=z\right)
=
p\!\left(y \mid A=0,\,E=e,\,Z=z\right).
\]
Substituting into the integral and using the definition of conditional
expectation yields
\[
\mu(e,z)
=
\mathbb{E}\!\left[\,Y_{\mathrm{obs}} \mid A=0,\,E=e,\,Z=z\,\right],
\]
as claimed.
\end{proof}

\subsection{Latent Recovery with VAEs}

We use a two-stage variational autoencoder to learn the content latents $Z$ and
the bias latent $A$. We write $q_\phi(\cdot)$ for encoder distributions
(variational posteriors) and $p(\cdot)$ for generative factors (priors and
decoders). Throughout, $Z$ is a continuous latent and $A \in \{0,1\}$ is a
binary bias indicator.

\paragraph{Stage~1: Learning Content Latents from Proxies.}
In the first stage, we learn the content latent $Z$ using only the proxy
measurements $Y_{\mathrm{proxy}}$ and the environment $E$. We specify an encoder
$q_{\phi_Z}(z \mid y_{\mathrm{proxy}}, E)$, an environment-conditioned prior
$p(z \mid E)$, and a decoder $p(y_{\mathrm{proxy}} \mid z)$. The resulting
evidence lower bound (ELBO) is
\[
\mathcal{L}_Z
=
\mathbb{E}_{q_{\phi_Z}}\!\bigl[\log p(y_{\mathrm{proxy}} \mid z)\bigr]
-
\mathrm{KL}\!\left(
  q_{\phi_Z}(z \mid y_{\mathrm{proxy}}, E)
  \,\Vert\,
  p(z \mid E)
\right).
\]
Because neither the observed outcome $Y_{\mathrm{obs}}$ nor the bias latent $A$
appears in this stage, and because the proxies are assumed not to depend on the
bias mechanism, the learned representation $Z$ captures variation associated
with the underlying content rather than reporting bias. Let $\hat z$ denote a point estimate of $Z$ obtained from encoder for an observation. In the second stage, $\hat z$ is treated as fixed.

\paragraph{Stage~2: Learning the Bias Latent.}
In the second stage, we infer the bias latent $A$ from the observed outcome
$Y_{\mathrm{obs}}$, conditional on the frozen content estimate $\hat z$ and the
environment $E$. We define an encoder
$q_{\phi_A}(a \mid Y_{\mathrm{obs}}, E, \hat z)$, an environment-conditioned prior
$p(a \mid E)$, and a decoder $p(Y_{\mathrm{obs}} \mid \hat z, a)$. The
corresponding ELBO is
\[
\mathcal{L}_A
=
\mathbb{E}_{q_{\phi_A}}\!\bigl[\log p(Y_{\mathrm{obs}} \mid \hat z, a)\bigr]
-
\mathrm{KL}\!\left(
  q_{\phi_A}(a \mid Y_{\mathrm{obs}}, E, \hat z)
  \,\Vert\,
  p(a \mid E)
\right).
\]
This stage attributes systematic errors in the observed outcome, relative to
the proxy-informed content representation, to the bias latent $A$, in
accordance with the assumed generative structure.

\paragraph{Combined encoder.}
Together, the two stages define an encoder that produces pointwise estimates of
the content and bias latents,
\[
(\hat z, \hat a)
=
\Bigl(
  \hat z,\;
  \hat a(Y_{\mathrm{obs}}, E, \hat z)
\Bigr),
\]
where $\hat z$ is obtained from the Stage~1 encoder and $\hat a$ from the
Stage~2 encoder. These latent estimates are subsequently used in a
post-hoc calibration step to quantify the effect of reporting bias.

\subsection{Bias Model and Estimation}

\begin{figure}[t]
\centering

\begin{minipage}[t]{0.62\linewidth}
\centering
\begin{tikzpicture}[
    >=stealth,
    n/.style={
        circle, draw,
        minimum size=10mm,
        inner sep=1pt,
        font=\scriptsize,
        text height=1.6ex,
        text depth=.25ex
    },
    obs/.style={n, fill=gray!20},
    latent/.style={n, fill=white},
    ell/.style={draw=none, fill=none, font=\scriptsize}
]

\node[obs]   (E)    at (0, 0.0)   {$E$};

\node[latent](Z)    at (-1.8,-2.0) {$Z$};
\node[latent](A)    at ( 1.8,-2.0) {$A$};

\node[obs]   (Yobs) at ( 1.8,-4.0) {$Y_{\mathrm{obs}}$};

\node[obs] (Yp1) at (-3.4,-4.9) {$Y_{\mathrm{pr}}^{(1)}$};
\node[ell] (Ypd) at (-1.8,-4.9) {\Large $\cdots$};
\node[obs] (Ypm) at (-0.2,-4.9) {$Y_{\mathrm{pr}}^{(m)}$};

\draw[->] (E) -- (Z);
\draw[->] (E) -- (A);

\draw[->] (Z) -- (Yp1);
\draw[->] (Z) -- (Ypm);

\draw[->] (Z) -- (Yobs);
\draw[->] (A) -- (Yobs);

\end{tikzpicture}

\vspace{0.4em}
\small\text{(a) Full generative model}
\end{minipage}
\hfill
\begin{minipage}[t]{0.34\linewidth}
\centering
\begin{tikzpicture}[
    >=stealth,
    obs/.style={circle, draw, minimum size=10mm, fill=gray!20, font=\scriptsize,
        text height=1.6ex, text depth=.25ex},
    latent/.style={circle, draw, minimum size=10mm, fill=white, font=\scriptsize,
        text height=1.6ex, text depth=.25ex}
]

\node[obs]    (E)     at (0,0)      {$E$};
\node[latent] (Z)     at (0,-2.3)   {$Z$};
\node[latent] (A)     at (3.3,-2.3) {$A$};

\node[latent] (Ytrue) at (0,-4.7)   {$Y_{\mathrm{true}}$};
\node[obs]    (Yobs)  at (3.3,-4.7) {$Y_{\mathrm{obs}}$};

\draw[->] (E) -- (Z);
\draw[->] (Z) -- (Ytrue);
\draw[->] (Ytrue) -- (Yobs);

\draw[->] (E) -- 
  node[pos=0.72, above, yshift=2pt, fill=white, inner sep=1pt, font=\scriptsize]
  {$A \in \{0,1\}$} (A);

\draw[->] (A) -- 
  node[midway, above, fill=white, inner sep=1pt, font=\scriptsize]
  {$+\;\alpha$} (Yobs);

\end{tikzpicture}

\vspace{0.4em}
\small\text{(b) Error model}
\end{minipage}

\caption{
\textbf{Causal structure for proxy-guided measurement calibration.}
(a) Full generative model: environment variables $E$ influence latent content factors $Z$
and latent reporting bias $A$. Proxy measurements $\{Y_1,\ldots,Y_m\}$ depend only on $Z$,
while the observed outcome $Y_{\mathrm{obs}}$ depends on both $Z$ and $A$.
(b) Error model highlighting the measurement bias mechanism: the true
outcome $Y_{\mathrm{true}}$ is perturbed by an environment-dependent binary bias $A$ with
additive magnitude $\alpha$.
}
\label{fig:pgmc_dag_combined}
\end{figure}
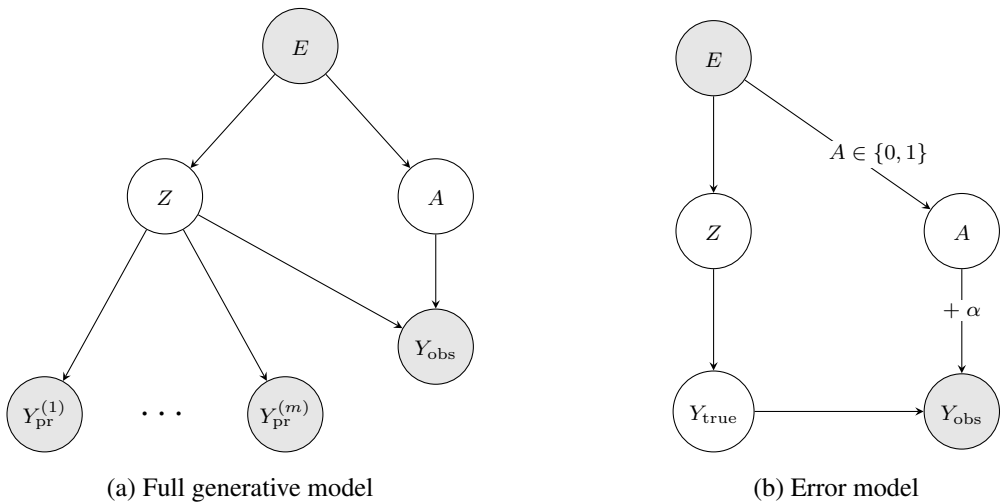

After recovering latent representations for content and bias, we focus on
estimating the magnitude of the reporting bias. We adopt a simple additive
bias model, illustrated in Figure~\ref{fig:pgmc_dag_combined}, in which the
observed outcome differs systematically from the true outcome when reporting
bias is present. The true outcome is generated as a function of the latent content and
environment,
\[
Y_{\mathrm{true}} = g(Z,E) + \varepsilon,
\qquad
\mathbb{E}[\varepsilon \mid Z,E] = 0,
\]
and the observed outcome incorporates a bias-induced shift,
\[
Y_{\mathrm{obs}} = Y_{\mathrm{true}} + \alpha\,A,
\qquad
A \in \{0,1\}.
\]
When $A=0$, the observed measurement is unbiased i.e. $Y_{\mathrm{obs}} = Y_{\mathrm{true}}$, while when $A=1$ the outcome is shifted by $\alpha$ on average. The scalar parameter $\alpha$ therefore quantifies the magnitude of reporting bias.

The two-stage VAE yields pointwise latent representations $(\hat Z,\hat A)$ for
each unit, where $\hat A$ is a real-valued latent score associated with the
binary bias indicator $A \in \{0,1\}$. Because latent representations are
identifiable only up to simple reparameterizations, the scale and orientation of
$\hat A$ are arbitrary. Consequently, numerical values of $\hat A$ (e.g.,
$0.9$ versus $0.1$) need not correspond in any fixed way to the presence or
absence of bias, and the mapping between $\hat A$ and $A$ may be reversed or
rescaled without affecting the implied data distribution.

To estimate the bias parameter $\alpha$, we construct a matched comparison
between units inferred to be in the biased regime and units inferred to be in
the unbiased regime. Let $\mathcal{I}_{1}$ denote the set of units with high
values of the latent bias score $\hat A$ and $\mathcal{I}_{0}$ the set of units
with low values of $\hat A$. For each unit $i \in \mathcal{I}_{1}$, we select a
set $\mathcal{N}_K(i) \subset \mathcal{I}_{0}$ consisting of its $K$ nearest
neighbors in the recovered content space $\hat Z$. The bias magnitude is estimated as
\[
\hat\alpha
=
\frac{1}{|\mathcal{I}_{1}|}
\sum_{i \in \mathcal{I}_{1}}
\left(
  Y_{\mathrm{obs},i}
  -
  \frac{1}{K}
  \sum_{j \in \mathcal{N}_K(i)} Y_{\mathrm{obs},j}
\right).
\]

This estimator contrasts observed outcomes for biased units with the average
outcomes of matched unbiased units that share similar latent content, thereby
removing variation attributable to the outcome function $g(Z,E)$. Under the additive bias model, and assuming that matching on $\hat Z$ (and on $E$
when included) renders the conditional mean of $g(Z,E)$ identical across matched
units, the estimator $\hat\alpha$ consistently recovers the average reporting
bias parameter $\alpha$.

\section{Experiments}

Evaluating outcome miscalibration in real-world settings is challenging because the true data-generating mechanisms are unknown and calibration errors are rarely documented. We therefore begin with a controlled synthetic setting, where the full latent structure is observable, and then validate our approach on semi-synthetic datasets built from randomized trials. Finally, we provide a real-world case study motivated by biases in recording natural disaster losses.

\begin{table}[!htbp]
\centering
\caption{Assumptions across experimental settings.
$\checkmark$ indicates the assumption is satisfied,
$\times$ indicates it is not assumed,
and $\circ$ indicates partial or indirect availability.}
\label{tab:assumptions}
\small
\begin{tabular}{l c c c}
\toprule
 & Synthetic & Semi-synthetic & Real-world \\
\midrule
(A1) Proxy exclusion
  & $\checkmark$ & $\checkmark$ & $\checkmark$ \\
(A2) Bias model: constant $\alpha$
  & $\checkmark$ & $\checkmark$ & $\times$ \\
(A3) Sign restriction: $\alpha \ge 0$
  & $\checkmark$ & $\checkmark$ & $\times$ \\
(A4) Linear structural mechanisms
  & $\checkmark$ & $\times$ & $\times$ \\
(A5) Availability of ground-truth $(A,Z)$
  & $\checkmark$ & $\circ$ & $\times$ \\
(A6) Overlap in $A$
  & $\checkmark$ & $\checkmark$ & $\circ$ \\
\bottomrule
\end{tabular}
\end{table}

\subsection{Experiments on Synthetic Data}

We first assess our framework under different synthetic data-generating processes (DGPs) that vary in latent dimensionality, sample size, and noise type. This allows us to examine conditions such as functional form, proxy informativeness, and bias strength that cannot be controlled in real-world observational data. Synthetic data is sampled from the following DGP:

\begin{align*}
\mathbf{e} &\sim \mathcal{N}(\mathbf{0}, \mathbf{I}_{d_e}),
\qquad \mathbf{e} \in \mathbb{R}^{d_e},
\\[6pt]
\mathbf{z} &= \mathbf{W}_z^\top \mathbf{e} + \boldsymbol{\varepsilon}_z,
\qquad \mathbf{W}_z \in \mathbb{R}^{d_e \times d_z},\;
\boldsymbol{\varepsilon}_z \sim \mathcal{N}(\mathbf{0}, \mathbf{I}_{d_z}),
\\[6pt]
A &= \mathbf{1}\!\left\{
\sigma\!\left(\mathbf{w}_a^\top \mathbf{e} + \varepsilon_a\right) > \tfrac12
\right\},
\qquad \mathbf{w}_a \in \mathbb{R}^{d_e},
\\[6pt]
Y_{\mathrm{true}} &= \mathbf{w}_y^\top \mathbf{z} + \varepsilon_y,
\qquad \mathbf{w}_y \in \mathbb{R}^{d_z},
\\[6pt]
Y_{\mathrm{obs}} &= Y_{\mathrm{true}} + \alpha A + \varepsilon_{\mathrm{obs}},
\qquad \alpha \in \mathbb{R},
\\[6pt]
Y_{\mathrm{proxy}}^{(k)} &= c_k\,Y_{\mathrm{true}} + \varepsilon_{\mathrm{proxy}}^{(k)},
\qquad k=1,\dots,m.
\end{align*}

In this synthetic setup, environment covariates
$\mathbf{e} = (E_1,\dots,E_{d_e})$ jointly induce latent content variables
$\mathbf{z} = (Z_1,\dots,Z_{d_z})$ and a scalar bias indicator $A$ through
linear functions of $\mathbf{e}$. The true outcome $Y_{\mathrm{true}}$ depends
only on the content latents, while the observed measurement $Y_{\mathrm{obs}}$
incorporates an additive bias of magnitude $\alpha$ whenever the bias
mechanism is active. Proxy measurements
$Y_{\mathrm{proxy}}^{(1)},\dots,Y_{\mathrm{proxy}}^{(m)}$ are generated as
independent noisy linear scalings of $Y_{\mathrm{true}}$ and are assumed to be
unaffected by the bias variable $A$, enforcing a strict exclusion restriction. This DGP makes several strong structural assumptions. In particular, both the
content and bias mechanisms are linear in the environment, the bias enters the
outcome additively, and all proxy variables are conditionally independent given
$Y_{\mathrm{true}}$ and depend on it only through fixed linear coefficients.

\paragraph{Setup.} 
Unless stated otherwise, we use $d_e = 10$ environment variables. The dataset is split into train/validation/test in an 80/10/10 ratio. Both VAE stages use Adam with learning rate $10^{-3}$. Following latent recovery, we estimate treatment effects via a matching estimator between treated units (high~$A$) and controls (low~$A$) stratified by~$Z$, where the threshold is chosen on the validation set. Evaluation focuses on estimating $\alpha$ and recovering $(Z,A)$ upto to permutation and scale.


\paragraph{Results.}
Table~\ref{tab:main-synth-alpha-estimation} reports $\hat{\alpha}$ across sample sizes, latent dimensions, and noise models. Across all settings, our method recovers $\alpha$ accurately, with performance improving as sample size grows. Bias magnitudes of $\alpha=5$ show the clearest trend with sample size, while extreme values ($\alpha=1$ and $\alpha=10$) exhibit broader variance. Gaussian and Poisson noise exhibit nearly identical behavior, indicating that the type of noise has no effect.

\begin{table}[!htbp]
\centering
\caption{Bias estimation with synthetic data under different settings.}
\label{tab:main-synth-alpha-estimation}
\small
\resizebox{0.8\textwidth}{!}{
\begin{tabular}{ccccccccc}
\toprule
\multicolumn{3}{c}{} & \multicolumn{6}{c}{Estimated $\alpha$} \\
\cmidrule(lr){4-9}
\multicolumn{3}{c}{} & \multicolumn{2}{c}{$\alpha=1$} & \multicolumn{2}{c}{$\alpha=5$} & \multicolumn{2}{c}{$\alpha=10$} \\
\cmidrule(lr){4-5}\cmidrule(lr){6-7}\cmidrule(lr){8-9}
$n$ & $d_z$ & $d_e$ & Gaussian & Poisson (scaled) & Gaussian & Poisson (scaled) & Gaussian & Poisson (scaled)  \\
\midrule
500 & 1 & 10 & 1.20±0.23 & 1.26±0.23 & 4.16±0.65 & 4.17±0.64 & 8.59±1.48 & 8.61±1.47 \\
500 & 2 & 10 & 1.07±0.28 & 1.06±0.36 & 3.88±0.51 & 3.63±0.48 & 8.23±1.47 & 7.82±1.65 \\
500 & 5 & 10 & 1.08±0.22 & 0.88±0.33 & 3.95±0.48 & 3.71±0.65 & 8.61±0.93 & 8.19±0.93 \\
500 & 10 & 10 & 1.01±0.28 & 1.06±0.15 & 3.81±1.37 & 4.03±1.07 & 7.58±2.31 & 8.16±1.56 \\
500 & 20 & 10 & 1.28±0.32 & 1.21±0.34 & 3.90±0.30 & 3.72±0.41 & 8.26±2.38 & 8.33±2.34 \\
\midrule
1000 & 1 & 10 & 1.18±0.21 & 1.16±0.22 & 4.34±0.34 & 4.37±0.50 & 9.02±0.97 & 8.93±1.21 \\
1000 & 2 & 10 & 0.89±0.12 & 0.86±0.13 & 3.94±0.47 & 4.03±0.50 & 9.23±0.97 & 8.89±1.07 \\
1000 & 5 & 10 & 0.76±0.17 & 0.76±0.16 & 4.03±0.41 & 3.95±0.37 & 8.86±1.22 & 8.69±1.01 \\
1000 & 10 & 10 & 1.04±0.16 & 1.05±0.15 & 4.41±0.47 & 4.31±0.47 & 8.04±1.78 & 7.91±1.48 \\
1000 & 20 & 10 & 1.29±0.17 & 1.28±0.18 & 3.64±0.25 & 3.66±0.30 & 8.34±1.20 & 8.36±1.20 \\
\midrule
2500 & 1 & 10 & 0.97±0.11 & 0.98±0.13 & 4.01±0.43 & 4.02±0.46 & 7.93±1.04 & 7.97±1.07 \\
2500 & 2 & 10 & 0.80±0.18 & 0.83±0.22 & 4.31±0.44 & 4.18±0.56 & 8.78±0.78 & 8.74±0.80 \\
2500 & 5 & 10 & 0.74±0.15 & 0.77±0.12 & 4.43±0.35 & 4.43±0.35 & 8.83±0.70 & 8.81±0.59 \\
2500 & 10 & 10 & 0.96±0.07 & 0.97±0.10 & 4.44±0.44 & 4.33±0.52 & 8.90±1.07 & 8.77±1.26 \\
2500 & 20 & 10 & 1.13±0.25 & 1.09±0.25 & 4.01±0.26 & 3.98±0.24 & 8.87±1.10 & 9.01±0.98 \\
\midrule
5000 & 1 & 10 & 1.03±0.16 & 1.02±0.14 & 4.34±0.50 & 4.24±0.55 & 8.46±1.31 & 8.33±1.43 \\
5000 & 2 & 10 & 0.84±0.16 & 0.84±0.13 & 4.28±0.54 & 4.32±0.51 & 8.81±1.28 & 8.81±1.18 \\
5000 & 5 & 10 & 0.74±0.09 & 0.75±0.07 & 4.50±0.69 & 4.50±0.68 & 9.08±1.35 & 8.99±1.38 \\
5000 & 10 & 10 & 0.95±0.06 & 0.93±0.06 & 4.54±0.39 & 4.41±0.42 & 9.04±0.86 & 8.83±0.79 \\
5000 & 20 & 10 & 0.98±0.15 & 1.02±0.18 & 4.19±0.45 & 4.20±0.42 & 8.14±1.26 & 8.20±1.16 \\
\midrule
10000 & 1 & 10 & 1.11±0.08 & 1.10±0.06 & 4.39±0.57 & 4.43±0.56 & 8.44±1.04 & 8.70±1.11 \\
10000 & 2 & 10 & 0.98±0.06 & 0.97±0.05 & 4.10±0.60 & 4.14±0.52 & 8.69±1.17 & 8.72±1.09 \\
10000 & 5 & 10 & 0.97±0.17 & 0.96±0.09 & 4.29±0.64 & 4.46±0.58 & 8.78±1.08 & 9.20±0.88 \\
10000 & 10 & 10 & 1.05±0.08 & 1.05±0.06 & 3.91±0.67 & 4.14±0.62 & 7.89±1.54 & 8.24±1.28 \\
10000 & 20 & 10 & 1.18±0.12 & 1.20±0.12 & 4.81±0.09 & 4.65±0.35 & 8.79±1.33 & 8.90±1.11 \\
\midrule
\bottomrule
\end{tabular}
}
\end{table}

\subsection{Experiments on Semi-Synthetic Data}

To evaluate performance in more realistic settings, we next construct semi-synthetic datasets from two randomized controlled trials. In these datasets, the environment variables and proxy measurements, and consequently the latents, arise naturally from the real study context, while reporting bias is introduced artificially. This setting allows us to study real-world latent mechanisms while having access to the induced bias level ~$\alpha$.

\paragraph{Data.}
We evaluate our method using two standard causal inference benchmarks. The Oregon Health Insurance Experiment (OHIE) is a randomized Medicaid lottery in which treatment corresponds to insurance enrollment. The outcome is log-transformed total hospital spending, and we use five biomedical and mental health indicators as proxy measurements that reflect underlying health status but are not subject to reporting bias. Environment variables capture pre-lottery utilization and hospitalization history. The JOBS dataset studies the impact of a job training program, with treatment defined by program participation. The outcome is 1978 earnings, while pre-treatment earnings from 1974 and 1975 serve as unbiased proxies. Environment covariates include standard demographic attributes such as age, education, race, and marital status. Both datasets provide clean treatment assignment with auxiliary information that can be used as proxies, making them well-suited for evaluating proxy-guided measurement calibration.



\paragraph{Experimental Setup.}
We compare our method against several baselines that do not explicitly leverage proxies, all evaluated under $k$-fold cross-validation
over 10 folds. In all settings, folds are constructed at the unit level, models are
trained exclusively on training folds, and all latent representations and
estimators are frozen before evaluation on held-out test folds. 

The \emph{proxy-only} baseline uses the proxy variables to learn a predictor
$\widehat{Y}$ of the observed outcome, treating proxies as unbiased measurements
of the latent signal. Bias is then estimated by contrasting residuals
$Y_{\mathrm{obs}}-\widehat{Y}$ across inferred treatment groups, attributing
systematic residual differences to reporting distortion. The \emph{environment-only} baseline models $Y_{\mathrm{obs}}$ directly as a
function of environmental covariates, attributing systematic variation across
treatment groups to environment-driven reporting effects. Differences predicted
by this model are used to estimate the bias magnitude $\hat{\alpha}$, implicitly
assuming that environmental variables capture the dominant sources of reporting
heterogeneity.


We also include TEDVAE \citet{zhang2021treatment}, a variational autoencoder designed to disentangle latent factors relevant for heterogeneous treatment effect estimation using proxies for unobserved confounders. In contrast, our method uses a two-stage VAE to disentangle content and bias
latents by combining proxy exclusion with environment-conditioned bias modeling
prior to estimating $\alpha$.

\paragraph{Results.}
Table~\ref{tab:main-synth-alpha-estimation} shows differences across
datasets. On OHIE, our proxy-guided model accurately recovers the true bias
magnitude $\alpha$ across all regimes and substantially outperforms all baselines,
with particularly strong performance for $\alpha=5$ and $\alpha=10$. On JOBS,
our method consistently improves over proxy-only, environment-only, and TEDVAE
baselines, but underestimates $\alpha$ in moderate and high bias settings,
reflecting the increased difficulty of calibration in this dataset. Proxy-only and
environment-only baselines substantially overestimate bias in both datasets,
while TEDVAE often yields near-zero estimates across regimes. TEDVAE optimizes a latent noise component to support treatment effect estimation rather than to identify or preserve the magnitude of systematic measurement bias, which leads to attenuated bias estimates in calibration-focused evaluations. Across both datasets, our performance is moderately stable with respect to the latent dimension $d_z$, indicating robustness to the choice of representation size.

\begin{table*}[t]
\centering
\caption{Estimated $\hat\alpha$ for JOBS and OHIE across baselines and latent dimensions.}
\small
\resizebox{\textwidth}{!}{
\begin{tabular}{l | c c c | c c c}
\toprule
 & \multicolumn{3}{c|}{\textbf{JOBS}} & \multicolumn{3}{c}{\textbf{OHIE}} \\
\cmidrule(lr){2-4} \cmidrule(lr){5-7}
 & $\alpha=1$ & $\alpha=5$ & $\alpha=10$ & $\alpha=1$ & $\alpha=5$ & $\alpha=10$ \\
\midrule
Baseline (Proxy only) & 4.718 ± 0.03 & 8.382 ± 0.04 & 13.382 ± 0.04 & 22.418 ± 0.05 & 24.073 ± 0.06 & 27.973 ± 0.07 \\
Baseline (Env only) & 3.526 ± 0.13 & 7.010 ± 0.26 & 11.979 ± 0.26 & 15.204 ± 0.55 & 15.200 ± 0.47 & 16.114 ± 0.29 \\
Baseline (TEDVAE) & 0.505 ± 0.32 & 1.304 ± 0.71 & 2.275 ± 1.37 & 0.225 ± 0.18 & 0.279 ± 0.18 & 0.769 ± 0.43 \\
Ours dim. (Z=5) & 1.694 ± 0.42 & 3.807 ± 1.14 & 7.580 ± 1.30 & 1.807 ± 0.37 & 4.415 ± 0.30 & 7.700 ± 0.36 \\
Ours dim. (Z=10) & 1.505 ± 0.35 & 3.529 ± 0.89 & 6.518 ± 1.63 & 1.941 ± 0.36 & 4.320 ± 0.18 & 7.548 ± 0.28 \\
Ours dim. (Z=15) & 1.498 ± 0.44 & 3.583 ± 0.71 & 6.773 ± 1.69 & 1.690 ± 0.37 & 4.311 ± 0.31 & 7.594 ± 0.50 \\
Ours dim. (Z=20) & 1.556 ± 0.46 & 3.296 ± 0.47 & 7.647 ± 1.85 & 1.863 ± 0.36 & 4.294 ± 0.14 & 7.832 ± 0.39 \\
\bottomrule
\end{tabular}}
\end{table*}


\subsection{Case Study on Real-World Data: SHELDUS Disaster Loss Data}

For a real-world case study, we use SHELDUS, a disaster loss database \citep{SHELDUS2025}. SHELDUS reports the county-level damages in terms of property and crop damages, along with injury and fatality counts for natural disasters.

\paragraph{Data.}
Our units of interest include counties in SHELDUS that have been affected by wildfires, hurricanes/tropical storms, flooding, and tornadoes. As our outcome variable of interest, we chose the loss for property damage in the county. The observed outcome $Y_{\mathrm{obs}}$ is log property damage for each
county--hazard--year. Remote sensing indicators, which provide information on land cover change from one class to another (e.g., area change from built-up to water indicative during floods), form a set of proxy variables. We limit the time period of interest from 2016 to 2023, as the remote sensing product we rely on, Dynamic World, is available only from 2015 \citep{brown2022dynamic}. We provide more information on the dataset in the appendix \ref{app:semi}.  Environment variables
$E$ includes demographic and socioeconomic attributes (e.g., population, income,
poverty rate, median age) and the category of the disaster.


\paragraph{Estimating effects of reporting bias.} Our framework enables counterfactual comparisons between counties that
are physically similar but differ in their degree of reporting bias.
Counties are partitioned into ``treated'' (high bias) and ``control'' (low bias)
groups using a threshold on the estimated latent ~$\hat{A}$. For each treated county~$i$, we estimate a counterfactual outcome by matching to
its $K$ nearest control counties in the latent space~$Z$. The resulting
conditional average treatment effect is estimated as
\[
\widehat{\tau}_i
=
\left|
Y_{\mathrm{obs},i}
-
\frac{1}{K}
\sum_{j \in \mathcal{N}_K(i)}
Y_{\mathrm{ctrl},j}
\right|.
\]
This quantity represents a local estimate of the conditional average treatment
effect evaluated at the county’s latent content representation $Z_i$, and
quantifies the magnitude of reporting bias while holding underlying physical
characteristics fixed.

\paragraph{Results.}
Figure~\ref{fig:county-cate-maps} reveals pronounced geographic heterogeneity in
reporting bias across counties. Counties with no qualifying events appear in
white. The county-level maps in Figure~\ref{fig:county-cate-maps} indicate that
hurricane-related reporting bias is concentrated along the coastal regions, particularly
in Florida, where we might see a direct landfall. In contrast, we do not see a clear cluster of biased counties for other events. However, patches of hotspots, such as California for wildfires and the Tornado Alley in the central U.S for tornadoes, appear to yield low estimates for bias. 

Figure~\ref{fig:hazard-bar} aggregates local CATE magnitudes by hazard type.
Floods exhibit the largest average magnitude of reporting bias, followed by
tornadoes, whereas wildfires and hurricanes show comparatively lower effects of reporting bias. This ordering is consistent with prior analyses based on SHELDUS and related loss databases, which document greater uncertainty in flood loss reporting relative to other events \citep{gall2009losses,zhou2025knowing}.

\begin{figure}[t]
  \centering

  \begin{minipage}[t]{0.5\linewidth}
    \centering
    \includegraphics[width=\linewidth]{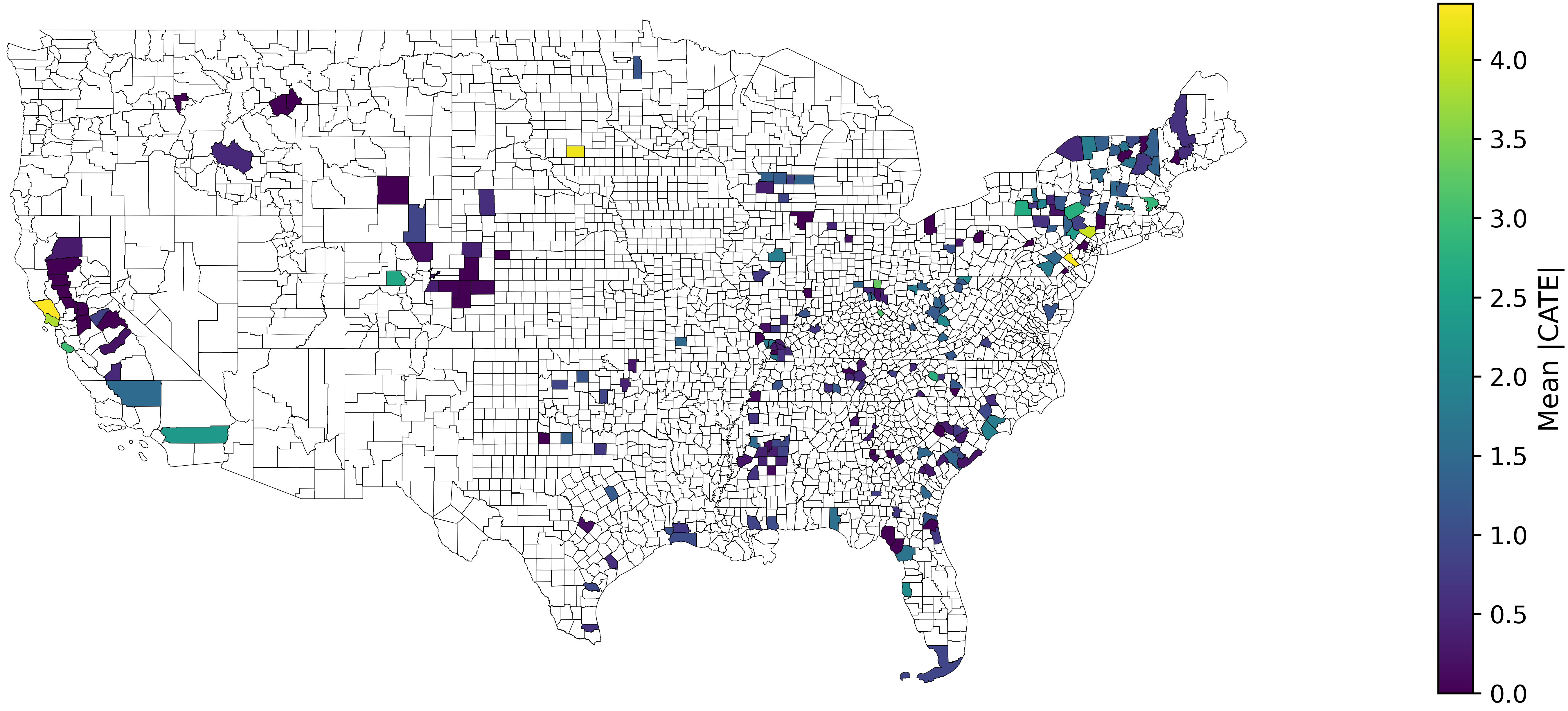}\\[-2pt]
    {\small Flood}
  \end{minipage}\hfill
  \begin{minipage}[t]{0.5\linewidth}
    \centering
    \includegraphics[width=\linewidth]{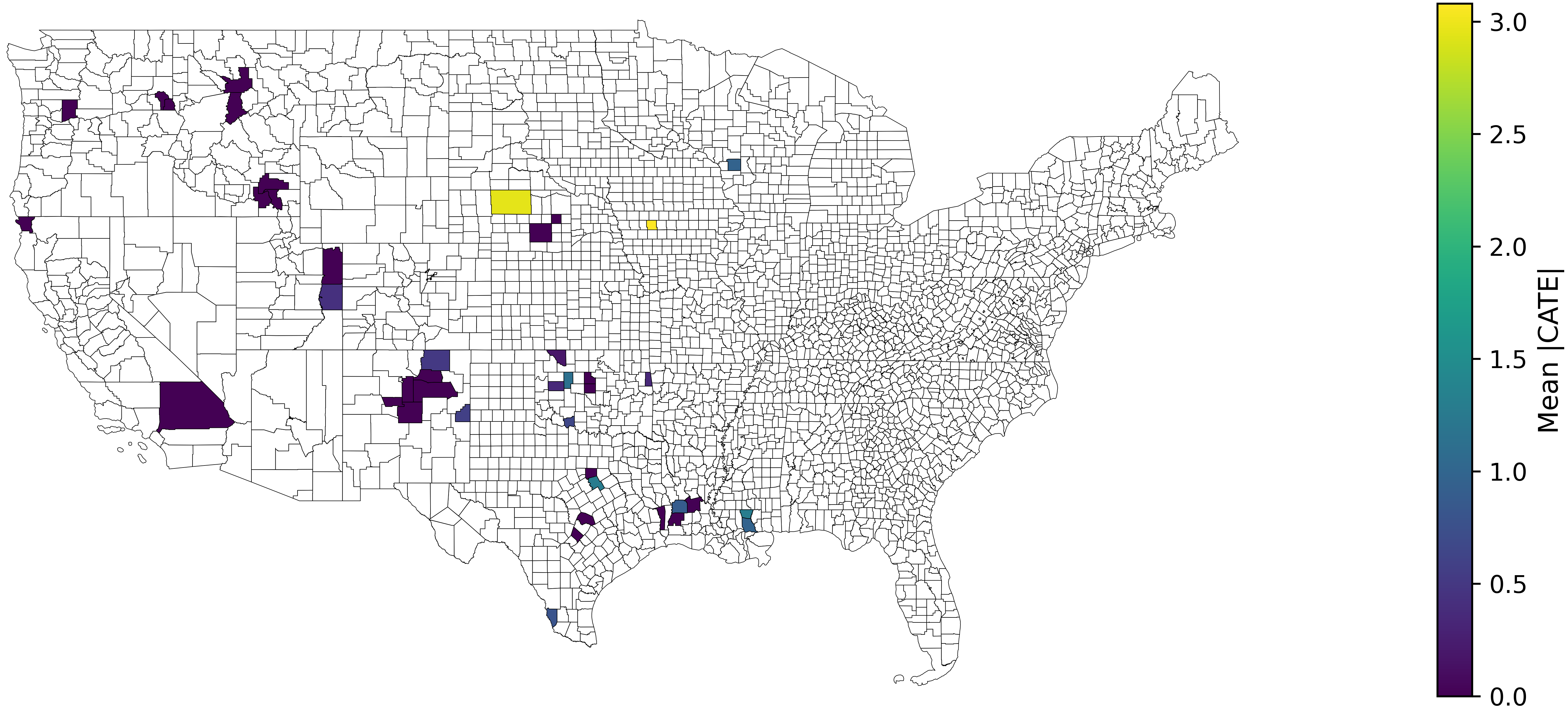}\\[-2pt]
    {\small Wildfire}
  \end{minipage}

  \vspace{6pt}

  \begin{minipage}[t]{0.5\linewidth}
    \centering
    \includegraphics[width=\linewidth]{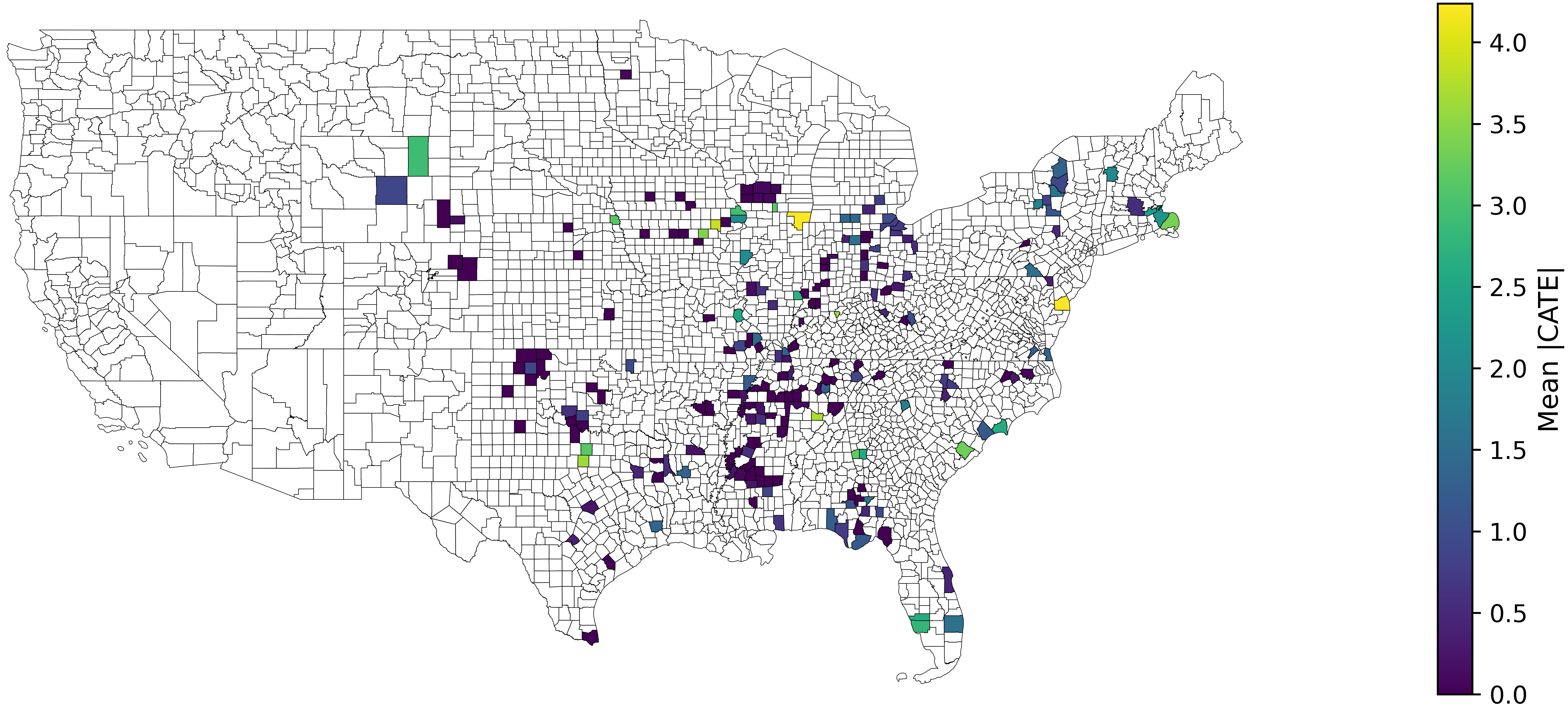}\\[-2pt]
    {\small Tornado}
  \end{minipage}\hfill
  \begin{minipage}[t]{0.5\linewidth}
    \centering
    \includegraphics[width=\linewidth]{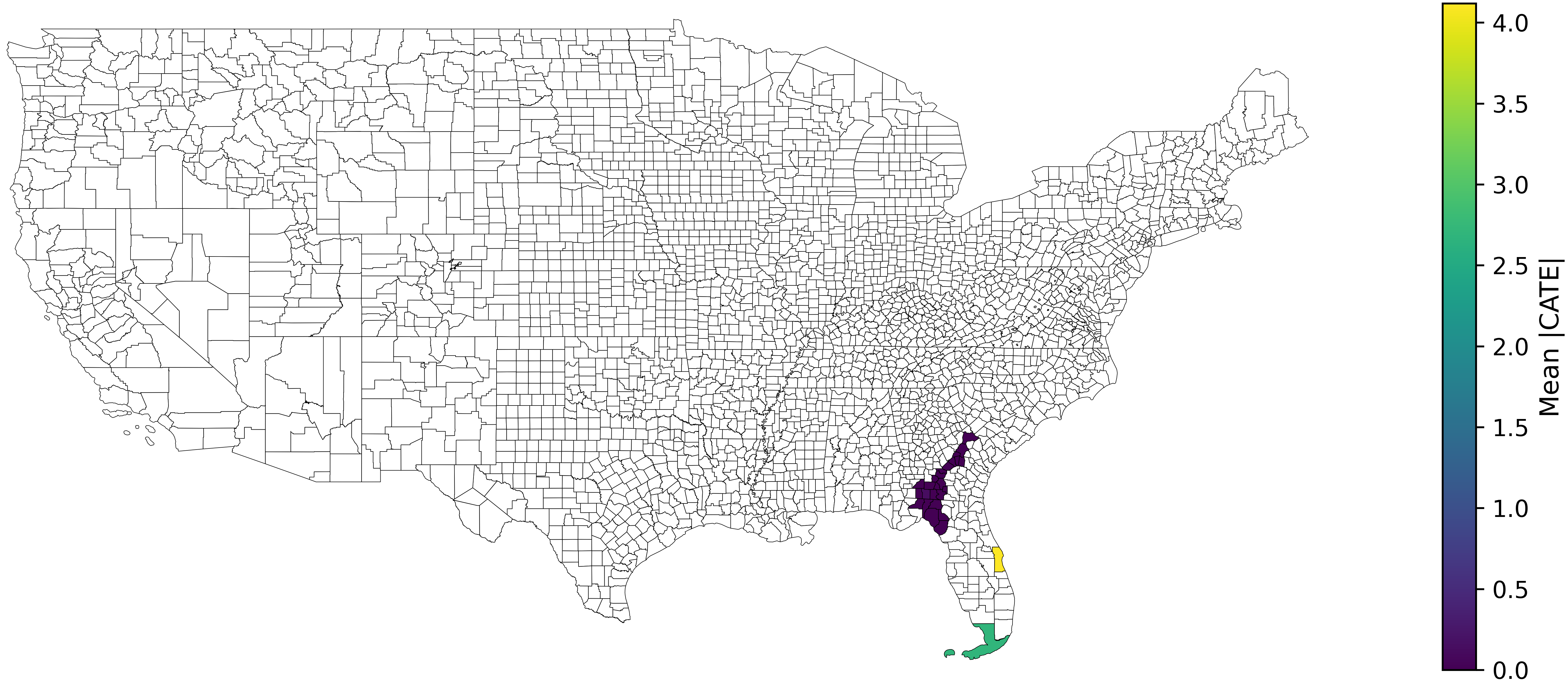}\\[-2pt]
    {\small Hurricane}
  \end{minipage}

  \caption{County-level mean absolute CATE estimates $|\widehat{\tau}_i|$ for 2023 across four hazard types. White indicates counties with no event; lighter colors represent higher estimated reporting bias.}
  \label{fig:county-cate-maps}
\end{figure}

\begin{figure}[t]
    \centering
    \includegraphics[width=0.8\linewidth]{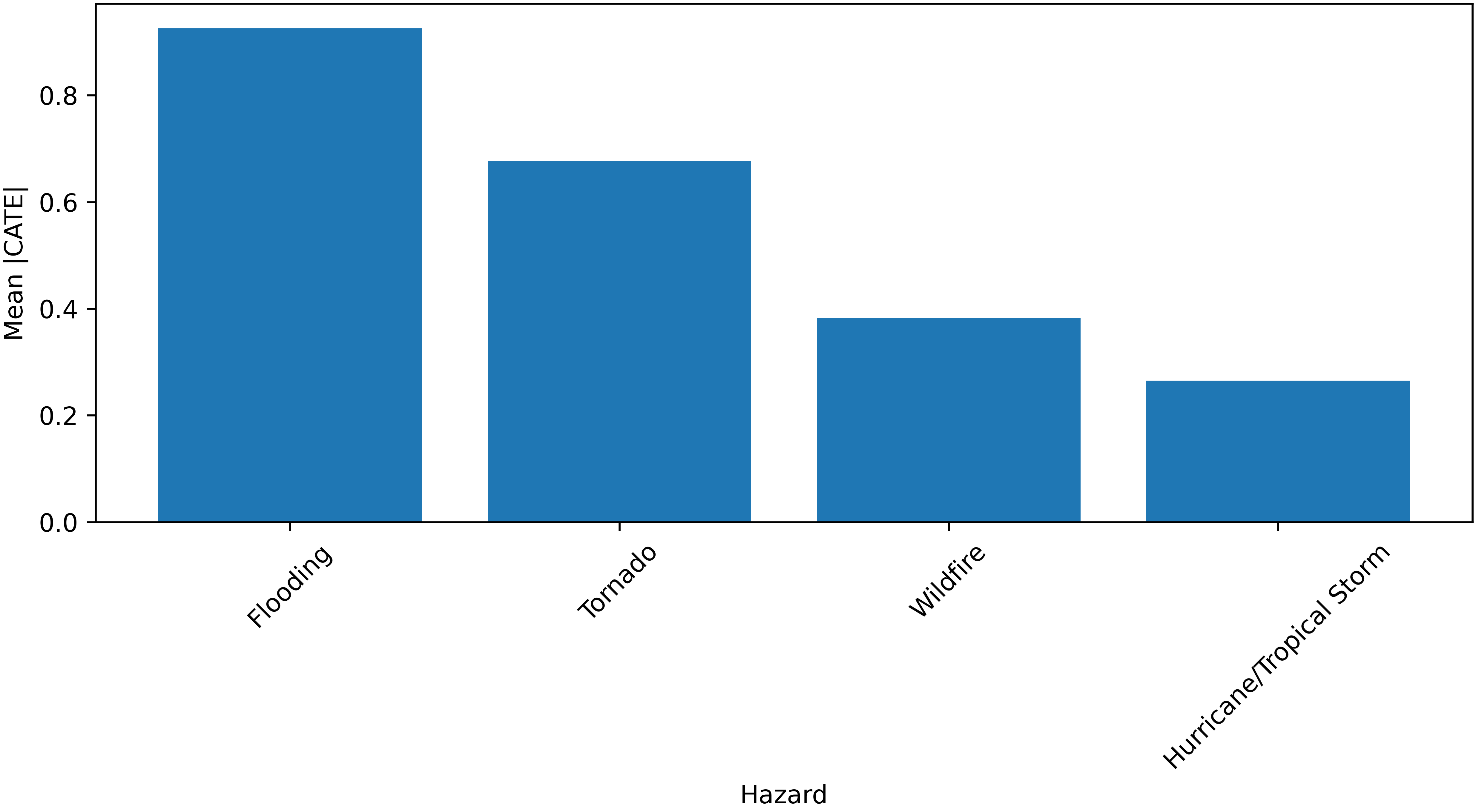}
    \caption{Mean absolute CATE values grouped by hazard type. Wildfire events
    exhibit the highest systematic reporting distortion, followed by tornadic
    and flooding events.}
    \label{fig:hazard-bar}
\end{figure}

\section{Discussion}

This work introduces proxy-guided measurement calibration, a framework for recovering true outcomes from systematically biased measurements. We formalize this problem through an explicit causal model in which measurement distortion is driven by latent bias factors that act on the observed outcome but are excluded from proxy variables. By separating the latent space into content-specific latents and bias-specific latents, where proxy variables depend only on the former, we provide a principled approach to identifying the magnitude of systematic miscalibration. Across different synthetic and semi-synthetic settings, we empirically show the performance of our framework, followed by a real-world case study.

Our semi-parametric approach, which combines co-trained latent representations with a matching-based non-parametric estimator, demonstrates a decent performance across a range of settings. Although the learned latent variables are identifiable only up to scale and permutation, this indeterminacy does not affect estimation because matching is performed exclusively in the latent space, while outcomes are compared on their original scale. As a result, the causal estimand remains invariant to latent rescaling or rotation. In contrast to parametric methods that propagate latent rescaling into outcome predictions, our estimator avoids outcome rescaling issues, leading to improved empirical performance.

Several limitations provide an avenue for future work. Firstly, the assumed error model remains restrictive. Although the real-world application relaxes several assumptions imposed in the synthetic and semi-synthetic settings, it still relies on the monotonicity assumption. Relaxing this assumption for identifying the magnitude of the bias is an important direction for future work. Moreover, our framework identifies conditional average treatment effects at the unit level, rather than individual treatment effects, reflecting fundamental limits imposed by the available data and assumptions. Our work finds many use cases where we can easily obtain proxy variables to correct for the data-generating bias. Extending this framework to other domains where outcomes are systematically mismeasured, such as public health surveillance, administrative records, and environmental monitoring, remains an important direction for future research.

\section{Acknowledgements}

This material is based upon work supported by, or in part by the U.S. Army Materiel Command under Grant Award Number W911NF24-2-0175 and by the U.S. Army Research Laboratory under Grant Award Number W911NF2020124. The views and conclusions contained in this document are those of the authors and should not be interpreted as representing the official policies of the U.S. Army Materiel Command or the U.S. Army Research Laboratory.

\bibliography{references}

\appendix
\clearpage

\makeatletter
\global\let\old@starttoc\@starttoc
\def\@starttoc#1{\ifx#1toc\global\let\@starttoc\old@starttoc\def\contentsline##1##2##3{}\fi\old@starttoc{#1}}
\makeatother

\begingroup
\let\oldaddcontentsline\addcontentsline
\renewcommand{\addcontentsline}[3]{%
  \ifx#1\toc
    \oldaddcontentsline{toc}{##2}{##3}%
  \fi
}

\tableofcontents
\endgroup
\clearpage

\section{Additional Experimental Discussion}
\label{app:experiments}

\subsection{Experiments on Synthetic Data}
\label{app:synth}
Before evaluating proxy-guided measurement calibration, we validate that the synthetic datasets conform to the intended causal data-generating process. These checks ensure that (i) the imposed structural assumptions are respected, (ii) the injected bias behaves as designed, and (iii) the latent variables are recoverable up to the expected indeterminacies. Together, these diagnostics establish that performance results in the main experiments reflect properties of the method rather than artifacts of data construction.

Tables~\ref{tab:a_mae_table}--\ref{tab:z_mae_table} and
Figures~\ref{fig:page3_scatter}--\ref{fig:page1_dists} summarize the synthetic
data validation results. The bias latent error in
Table~\ref{tab:a_mae_table} increases with the injected bias strength~$\alpha$,
reflecting that the inferred bias representation is unconstrained and captures
bias magnitude rather than acting as a bounded classifier. In contrast,
Table~\ref{tab:rmse_proxies_table} shows that proxy reconstruction error remains
stable across $\alpha$ and noise models, confirming that proxies depend only on
the content latents and are unaffected by bias. Table~\ref{tab:z_mae_table}
demonstrates consistent recovery of the content latents up to scale and
permutation. These quantitative results are complemented by distributional diagnostics in
Figures~\ref{fig:page1_dists} and~\ref{fig:page3_metrics}, as well as qualitative
latent alignment visualizations in
Figure~\ref{fig:page3_scatter}.

\begin{figure}[!htbp]
    \centering
    \includegraphics[width=\textwidth]{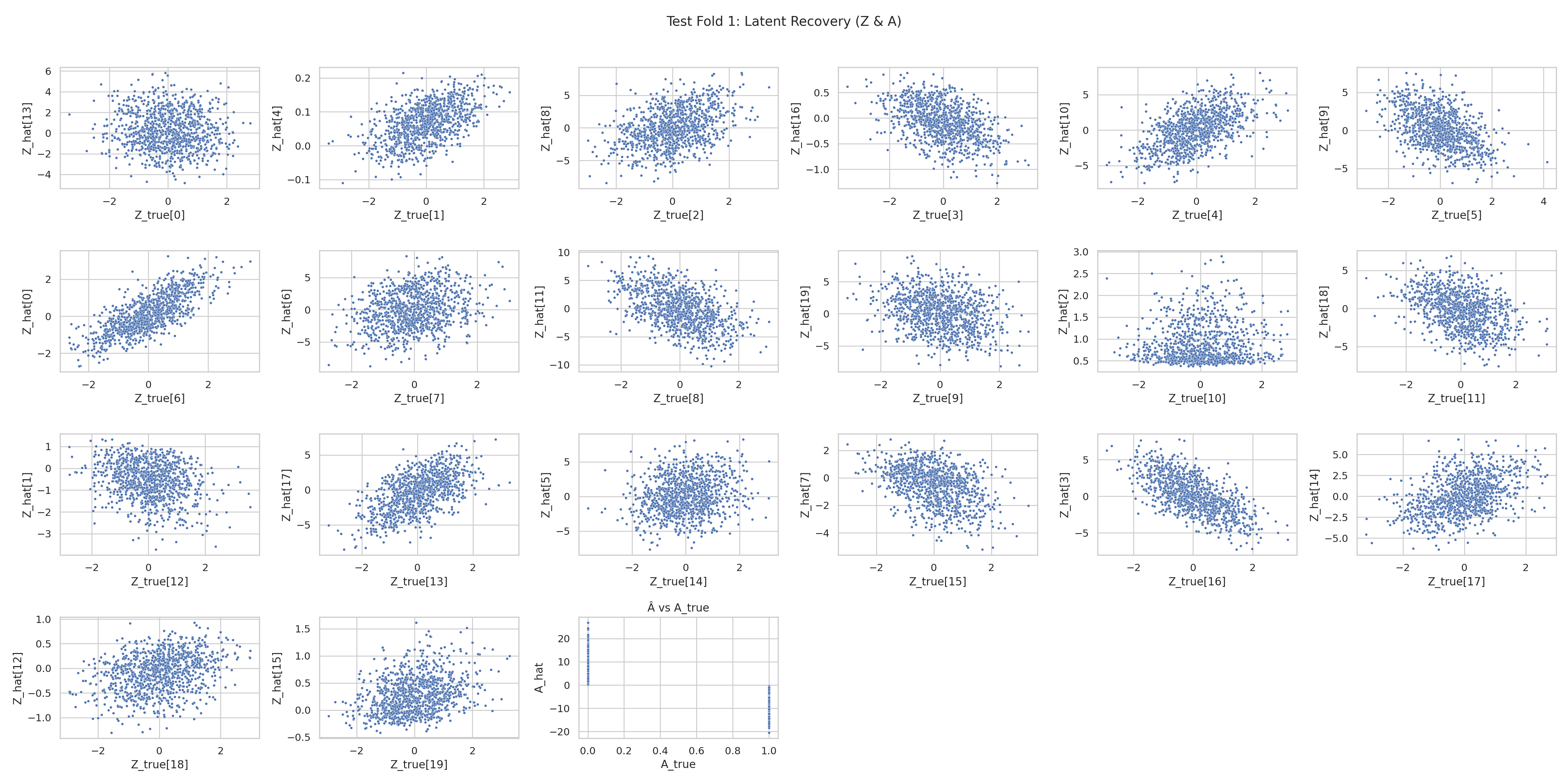}
    \caption{
Scatter plot of the recovered content latents $\hat Z$ versus the true latents $Z$
for a representative fold, after aligning dimensions using the closest permutation that maximizes permuted $R^2$.
    }
    \label{fig:page3_scatter}
\end{figure}

\begin{figure}[t]
    \centering
    \includegraphics[
        width=\linewidth,
        height=0.9\textheight,
        keepaspectratio
    ]{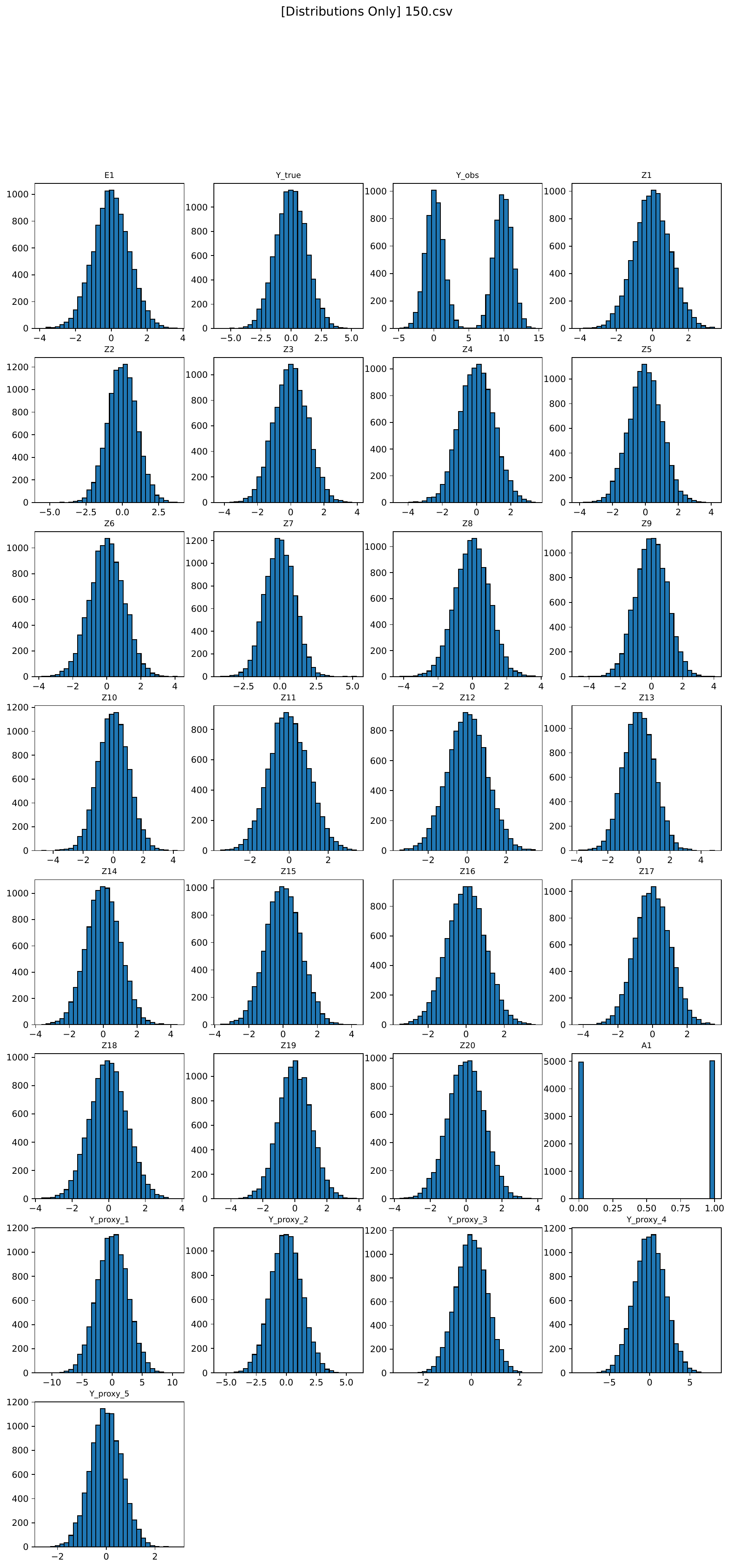}
    \caption{
        \textbf{Sanity Check for the Distribution.}
        Histograms for all key variables: $E_1$, $Z$, $A$, $Y_{\text{true}}$, 
        $Y_{\text{obs}}$, and proxies for a given configuration. Ensures data follows the intended generative model.
    }
    \label{fig:page1_dists}
\end{figure}

\begin{figure}[t]
    \centering
    \includegraphics[width=\textwidth]{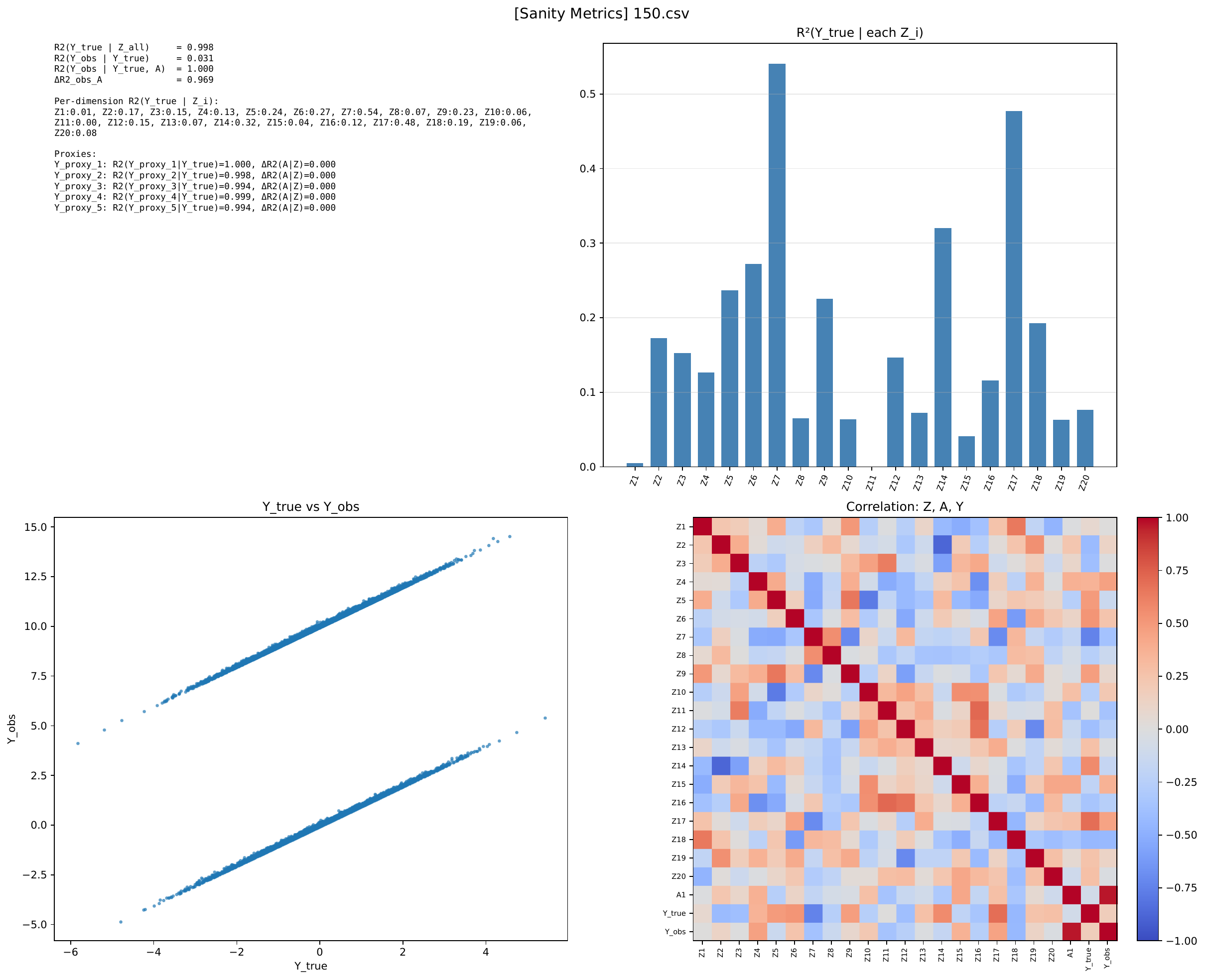}
    \caption{
        \textbf{Sanity Check - Core Metrics and Correlations.}
        This page summarizes the main structural sanity metrics:
        $R^2(Y_{\text{true}} \mid Z)$, $R^2(Y_{\text{obs}} \mid Y_{\text{true}})$,
        the effect of bias latents via $\Delta R^2$, per-dimension $R^2$ for each $Z_i$,
        proxy alignment metrics, as well as the correlation matrix over $\{Z, A, Y\}$.
        These diagnostics confirm that the dataset correctly encodes both latent content
        and latent bias structure.
    }
    \label{fig:page3_metrics}
\end{figure}

\begin{table*}[!htbp]
\centering
\caption{Mean$\pm$Std of $a_{mae}$ grouped by $(n,d_z,d_e)$ rows and $(\alpha,\text{noise})$ columns.}
\label{tab:a_mae_table}
\small
\resizebox{\textwidth}{!}{
\begin{tabular}{ccccccccc}
\toprule
\multicolumn{3}{c}{} & \multicolumn{6}{c}{$a_{mae}$} \\
\cmidrule(lr){4-9}
\multicolumn{3}{c}{} & \multicolumn{2}{c}{$\alpha=1$} & \multicolumn{2}{c}{$\alpha=5$} & \multicolumn{2}{c}{$\alpha=10$} \\
\cmidrule(lr){4-5}\cmidrule(lr){6-7}\cmidrule(lr){8-9}
$n$ & $d_z$ & $d_e$ & Gaussian & Poisson (scaled) & Gaussian & Poisson (scaled) & Gaussian & Poisson (scaled) \\
\midrule
500 & 1 & 10 & 1.29±0.33 & 1.29±0.32 & 1.39±0.47 & 1.38±0.49 & 1.48±0.50 & 1.46±0.51 \\
500 & 2 & 10 & 1.52±0.22 & 1.54±0.21 & 1.63±0.34 & 1.64±0.34 & 1.81±0.39 & 1.85±0.43 \\
500 & 5 & 10 & 1.37±0.27 & 1.29±0.38 & 1.09±0.31 & 1.18±0.28 & 1.19±0.29 & 1.13±0.37 \\
500 & 10 & 10 & 1.68±0.34 & 1.68±0.35 & 1.72±0.47 & 1.82±0.40 & 1.83±0.41 & 1.83±0.44 \\
500 & 20 & 10 & 1.64±0.32 & 1.63±0.22 & 1.57±0.39 & 1.57±0.39 & 1.51±0.61 & 1.66±0.43 \\
\midrule
1000 & 1 & 10 & 1.49±0.18 & 1.49±0.19 & 1.60±0.36 & 1.61±0.35 & 1.83±0.36 & 1.82±0.36 \\
1000 & 2 & 10 & 1.63±0.16 & 1.64±0.16 & 1.79±0.39 & 1.80±0.38 & 2.24±0.46 & 2.25±0.45 \\
1000 & 5 & 10 & 1.63±0.05 & 1.65±0.08 & 1.51±0.53 & 1.39±0.49 & 1.77±0.41 & 1.66±0.53 \\
1000 & 10 & 10 & 1.93±0.35 & 1.75±0.36 & 1.46±0.43 & 2.12±0.44 & 2.10±0.64 & 2.23±0.45 \\
1000 & 20 & 10 & 1.88±0.20 & 1.88±0.19 & 1.61±0.33 & 1.63±0.33 & 1.95±0.45 & 1.96±0.44 \\
\midrule
2500 & 1 & 10 & 1.75±0.29 & 1.77±0.31 & 2.12±0.48 & 2.10±0.47 & 2.59±0.48 & 2.60±0.49 \\
2500 & 2 & 10 & 1.90±0.13 & 1.89±0.14 & 2.73±0.35 & 2.67±0.38 & 3.30±0.41 & 3.22±0.40 \\
2500 & 5 & 10 & 1.70±0.50 & 1.87±0.19 & 2.74±0.55 & 2.18±1.00 & 2.96±0.53 & 2.76±0.96 \\
2500 & 10 & 10 & 2.22±0.37 & 2.21±0.39 & 3.05±0.53 & 3.09±0.46 & 3.24±0.53 & 3.10±0.88 \\
2500 & 20 & 10 & 2.16±0.22 & 2.17±0.20 & 2.07±0.46 & 2.09±0.46 & 3.04±0.58 & 3.03±0.58 \\
\midrule
5000 & 1 & 10 & 1.95±0.26 & 1.91±0.23 & 2.50±0.49 & 2.51±0.51 & 3.66±0.50 & 3.64±0.49 \\
5000 & 2 & 10 & 1.91±0.10 & 1.92±0.10 & 3.74±0.33 & 3.72±0.36 & 4.42±0.42 & 4.42±0.43 \\
5000 & 5 & 10 & 2.29±0.91 & 2.19±0.96 & 3.49±0.43 & 3.61±0.54 & 4.18±0.55 & 4.18±0.53 \\
5000 & 10 & 10 & 2.29±0.60 & 2.22±0.51 & 4.01±0.54 & 4.10±0.56 & 4.41±0.55 & 4.50±0.55 \\
5000 & 20 & 10 & 2.50±0.16 & 2.49±0.16 & 1.89±0.96 & 1.93±1.08 & 4.28±0.57 & 4.27±0.56 \\
\midrule
10000 & 1 & 10 & 1.13±0.17 & 1.14±0.18 & 1.86±0.72 & 1.85±0.71 & 4.23±0.58 & 4.20±0.61 \\
10000 & 2 & 10 & 1.24±0.07 & 1.22±0.06 & 4.07±0.53 & 4.15±0.49 & 5.44±0.31 & 5.40±0.33 \\
10000 & 5 & 10 & 1.58±1.50 & 1.14±0.09 & 3.92±1.13 & 4.17±1.18 & 5.40±0.61 & 4.90±1.38 \\
10000 & 10 & 10 & 1.30±0.35 & 1.41±0.45 & 5.12±0.55 & 5.32±0.78 & 6.46±1.06 & 6.13±0.58 \\
10000 & 20 & 10 & 1.83±0.15 & 1.82±0.15 & 1.02±0.59 & 1.13±0.72 & 5.42±0.52 & 5.41±0.58 \\
\midrule
\bottomrule
\end{tabular}
}
\end{table*}

\begin{table*}[!htbp]
\centering
\caption{Mean$\pm$Std of $rmse_{proxies}$ grouped by $(n,d_z,d_e)$ rows and $(\alpha,\text{noise})$ columns.}
\label{tab:rmse_proxies_table}
\small
\resizebox{\textwidth}{!}{
\begin{tabular}{ccccccccc}
\toprule
\multicolumn{3}{c}{} & \multicolumn{6}{c}{$rmse_{proxies}$} \\
\cmidrule(lr){4-9}
\multicolumn{3}{c}{} & \multicolumn{2}{c}{$\alpha=1$} & \multicolumn{2}{c}{$\alpha=5$} & \multicolumn{2}{c}{$\alpha=10$} \\
\cmidrule(lr){4-5}\cmidrule(lr){6-7}\cmidrule(lr){8-9}
$n$ & $d_z$ & $d_e$ & Gaussian & Poisson (scaled) & Gaussian & Poisson (scaled) & Gaussian & Poisson (scaled) \\
\midrule
500 & 1 & 10 & 0.20±0.03 & 0.20±0.03 & 0.20±0.03 & 0.20±0.03 & 0.20±0.03 & 0.20±0.03 \\
500 & 2 & 10 & 0.21±0.03 & 0.20±0.02 & 0.21±0.03 & 0.20±0.02 & 0.21±0.03 & 0.20±0.02 \\
500 & 5 & 10 & 0.19±0.02 & 0.19±0.02 & 0.19±0.02 & 0.19±0.02 & 0.19±0.02 & 0.19±0.02 \\
500 & 10 & 10 & 0.15±0.02 & 0.15±0.02 & 0.15±0.02 & 0.15±0.02 & 0.15±0.02 & 0.15±0.02 \\
500 & 20 & 10 & 0.26±0.03 & 0.25±0.03 & 0.26±0.03 & 0.25±0.03 & 0.26±0.03 & 0.25±0.03 \\
\midrule
1000 & 1 & 10 & 0.19±0.06 & 0.19±0.09 & 0.19±0.06 & 0.19±0.09 & 0.19±0.06 & 0.19±0.09 \\
1000 & 2 & 10 & 0.18±0.03 & 0.17±0.02 & 0.18±0.03 & 0.17±0.02 & 0.18±0.03 & 0.17±0.02 \\
1000 & 5 & 10 & 0.16±0.02 & 0.17±0.01 & 0.16±0.02 & 0.17±0.01 & 0.16±0.02 & 0.17±0.01 \\
1000 & 10 & 10 & 0.13±0.01 & 0.13±0.01 & 0.17±0.01 & 0.13±0.01 & 0.13±0.01 & 0.13±0.01 \\
1000 & 20 & 10 & 0.20±0.01 & 0.20±0.01 & 0.20±0.01 & 0.20±0.01 & 0.20±0.01 & 0.20±0.01 \\
\midrule
2500 & 1 & 10 & 0.14±0.01 & 0.14±0.01 & 0.14±0.01 & 0.14±0.01 & 0.14±0.01 & 0.14±0.01 \\
2500 & 2 & 10 & 0.14±0.01 & 0.14±0.01 & 0.14±0.01 & 0.15±0.01 & 0.14±0.01 & 0.15±0.01 \\
2500 & 5 & 10 & 0.13±0.01 & 0.14±0.01 & 0.13±0.01 & 0.14±0.01 & 0.14±0.01 & 0.14±0.01 \\
2500 & 10 & 10 & 0.12±0.00 & 0.12±0.00 & 0.12±0.00 & 0.12±0.00 & 0.12±0.01 & 0.12±0.00 \\
2500 & 20 & 10 & 0.17±0.01 & 0.16±0.01 & 0.17±0.01 & 0.17±0.01 & 0.17±0.01 & 0.17±0.01 \\
\midrule
5000 & 1 & 10 & 0.12±0.01 & 0.13±0.01 & 0.12±0.01 & 0.13±0.01 & 0.12±0.01 & 0.13±0.01 \\
5000 & 2 & 10 & 0.13±0.01 & 0.12±0.01 & 0.13±0.01 & 0.12±0.01 & 0.13±0.01 & 0.12±0.01 \\
5000 & 5 & 10 & 0.12±0.00 & 0.12±0.01 & 0.12±0.00 & 0.12±0.01 & 0.12±0.00 & 0.12±0.01 \\
5000 & 10 & 10 & 0.11±0.00 & 0.11±0.00 & 0.11±0.00 & 0.11±0.00 & 0.11±0.00 & 0.11±0.00 \\
5000 & 20 & 10 & 0.14±0.01 & 0.14±0.01 & 0.14±0.01 & 0.14±0.01 & 0.14±0.01 & 0.14±0.01 \\
\midrule
10000 & 1 & 10 & 0.11±0.01 & 0.11±0.01 & 0.11±0.01 & 0.11±0.01 & 0.11±0.01 & 0.11±0.01 \\
10000 & 2 & 10 & 0.11±0.01 & 0.11±0.01 & 0.11±0.01 & 0.11±0.01 & 0.11±0.01 & 0.11±0.01 \\
10000 & 5 & 10 & 0.10±0.00 & 0.10±0.00 & 0.10±0.00 & 0.10±0.00 & 0.10±0.00 & 0.10±0.00 \\
10000 & 10 & 10 & 0.09±0.00 & 0.09±0.00 & 0.09±0.00 & 0.09±0.00 & 0.09±0.00 & 0.09±0.00 \\
10000 & 20 & 10 & 0.11±0.00 & 0.11±0.00 & 0.11±0.00 & 0.11±0.00 & 0.11±0.00 & 0.11±0.00 \\
\midrule
\bottomrule
\end{tabular}
}
\end{table*}

\begin{table*}[!htbp]
\centering
\caption{Mean$\pm$Std of $z_{mae}$ grouped by $(n,d_z,d_e)$ rows and $(\alpha,\text{noise})$ columns.}
\label{tab:z_mae_table}
\small
\resizebox{\textwidth}{!}{
\begin{tabular}{ccccccccc}
\toprule
\multicolumn{3}{c}{} & \multicolumn{6}{c}{$z_{mae}$} \\
\cmidrule(lr){4-9}
\multicolumn{3}{c}{} & \multicolumn{2}{c}{$\alpha=1$} & \multicolumn{2}{c}{$\alpha=5$} & \multicolumn{2}{c}{$\alpha=10$} \\
\cmidrule(lr){4-5}\cmidrule(lr){6-7}\cmidrule(lr){8-9}
$n$ & $d_z$ & $d_e$ & Gaussian & Poisson (scaled) & Gaussian & Poisson (scaled) & Gaussian & Poisson (scaled) \\
\midrule
500 & 1 & 10 & 1.84±0.86 & 1.82±0.86 & 1.84±0.86 & 1.82±0.86 & 1.84±0.86 & 1.82±0.86 \\
500 & 2 & 10 & 1.42±0.41 & 1.50±0.37 & 1.42±0.41 & 1.50±0.37 & 1.42±0.41 & 1.50±0.37 \\
500 & 5 & 10 & 1.22±0.19 & 1.23±0.19 & 1.22±0.19 & 1.23±0.19 & 1.22±0.19 & 1.23±0.19 \\
500 & 10 & 10 & 1.14±0.06 & 1.14±0.07 & 1.14±0.06 & 1.14±0.07 & 1.14±0.06 & 1.14±0.07 \\
500 & 20 & 10 & 1.09±0.10 & 1.10±0.11 & 1.09±0.10 & 1.10±0.11 & 1.09±0.10 & 1.10±0.11 \\
\midrule
1000 & 1 & 10 & 2.07±0.92 & 2.08±0.94 & 2.07±0.92 & 2.08±0.94 & 2.07±0.92 & 2.08±0.94 \\
1000 & 2 & 10 & 1.59±0.48 & 1.59±0.48 & 1.59±0.48 & 1.59±0.48 & 1.59±0.48 & 1.59±0.48 \\
1000 & 5 & 10 & 1.29±0.20 & 1.29±0.21 & 1.29±0.20 & 1.29±0.21 & 1.29±0.20 & 1.29±0.21 \\
1000 & 10 & 10 & 1.22±0.08 & 1.22±0.08 & 1.12±0.06 & 1.22±0.08 & 1.22±0.08 & 1.22±0.08 \\
1000 & 20 & 10 & 1.18±0.10 & 1.17±0.11 & 1.18±0.10 & 1.17±0.11 & 1.18±0.10 & 1.17±0.11 \\
\midrule
2500 & 1 & 10 & 1.94±0.95 & 1.95±0.93 & 1.94±0.95 & 1.95±0.93 & 1.94±0.95 & 1.95±0.93 \\
2500 & 2 & 10 & 1.66±0.41 & 1.64±0.37 & 1.66±0.41 & 1.64±0.38 & 1.67±0.41 & 1.64±0.38 \\
2500 & 5 & 10 & 1.40±0.20 & 1.40±0.21 & 1.40±0.20 & 1.40±0.21 & 1.40±0.21 & 1.40±0.21 \\
2500 & 10 & 10 & 1.32±0.06 & 1.32±0.06 & 1.32±0.06 & 1.32±0.06 & 1.32±0.06 & 1.32±0.06 \\
2500 & 20 & 10 & 1.28±0.09 & 1.28±0.09 & 1.28±0.09 & 1.28±0.09 & 1.27±0.09 & 1.28±0.09 \\
\midrule
5000 & 1 & 10 & 2.24±0.82 & 2.23±0.83 & 2.24±0.82 & 2.23±0.83 & 2.24±0.82 & 2.23±0.83 \\
5000 & 2 & 10 & 1.75±0.48 & 1.75±0.48 & 1.75±0.48 & 1.75±0.48 & 1.75±0.48 & 1.75±0.48 \\
5000 & 5 & 10 & 1.49±0.23 & 1.49±0.23 & 1.49±0.23 & 1.49±0.23 & 1.49±0.23 & 1.49±0.23 \\
5000 & 10 & 10 & 1.44±0.16 & 1.44±0.16 & 1.44±0.16 & 1.44±0.16 & 1.44±0.16 & 1.44±0.16 \\
5000 & 20 & 10 & 1.46±0.10 & 1.46±0.09 & 1.46±0.10 & 1.46±0.09 & 1.46±0.10 & 1.46±0.09 \\
\midrule
10000 & 1 & 10 & 2.20±0.86 & 2.21±0.84 & 2.20±0.86 & 2.21±0.84 & 2.20±0.86 & 2.21±0.84 \\
10000 & 2 & 10 & 1.77±0.34 & 1.76±0.35 & 1.77±0.34 & 1.76±0.35 & 1.77±0.34 & 1.76±0.35 \\
10000 & 5 & 10 & 1.57±0.32 & 1.56±0.33 & 1.57±0.32 & 1.56±0.33 & 1.57±0.32 & 1.56±0.33 \\
10000 & 10 & 10 & 1.48±0.16 & 1.47±0.15 & 1.48±0.16 & 1.47±0.15 & 1.48±0.16 & 1.47±0.15 \\
10000 & 20 & 10 & 1.66±0.14 & 1.67±0.13 & 1.66±0.14 & 1.67±0.13 & 1.66±0.14 & 1.67±0.13 \\
\midrule
\bottomrule
\end{tabular}
}
\end{table*}

\clearpage

\subsection{Experiments on Semi-Synthetic Datasets}
\label{app:semi}
To evaluate proxy-guided calibration in settings that more closely resemble real-world predictive
tasks, we construct semi-synthetic datasets from two randomized controlled trials: the JOBS job
training experiment and the Oregon Health Insurance Experiment (OHIE). In both datasets, the
environment variables and proxy measurements arise naturally from high-quality experimental data,
while a controlled amount of reporting bias is injected synthetically. This design preserves the
advantages of real-world covariate structure and proxy behavior, while providing access to known
ground-truth bias levels for benchmarking.

\paragraph{JOBS (Employment Training).}
The JOBS dataset originates from the National Supported Work (NSW) program, later merged with the
PSID comparison sample to create a rich set of demographic and economic covariates.
Following the preprocessing in our pipeline, the outcome is 1978 earnings (log-transformed), and
the proxies consist of pretreatment earnings from 1974 and 1975, which are unaffected by reporting
bias and thus serve as clean measurements of underlying economic status.
Environment variables include treatment assignment and demographic attributes such as age,
education, race, marital status, and prior earnings. After preprocessing, the dataset contains a
fully numeric table with the structure shown in Table~\ref{tab:jobs_schema}. Bias is injected by
adding $\alpha A$ to the standardized outcome, where $A$ is a synthetically generated latent bias
variable derived only from environment covariates.

\paragraph{OHIE (Medicaid Lottery).}
The OHIE dataset is derived from the Oregon Medicaid expansion via a lottery-based allocation
mechanism. Treatment is the randomized selection to receive Medicaid coverage. The outcome is
log-transformed emergency department spending. Proxies represent biomedical and mental-health
measurements obtained during in-person assessments (blood pressure, HDL, A1C, and depression score),
all of which serve as clinical proxies for underlying health status. Environment covariates include
variables from the pre-randomization descriptive survey, state-level program participation, prior
emergency department visits, and selected in-person clinical measures. The preprocessing pipeline
filters environment variables by variance and removes collinear features, producing the unified
schema shown in Table~\ref{tab:ohie_schema}. As with JOBS, bias is introduced by adding
$\alpha A$ to the standardized observed outcome, where $A$ depends only on selected environment
variables.

Together, JOBS and OHIE provide complementary settings: an employment-focused economic context
with proxies reflecting income history, and a healthcare context with proxies reflecting physiological
health markers. These semi-synthetic datasets enable controlled evaluation of measurement error
methods while retaining realistic feature distributions, proxy relationships, and treatment assignment
from real experimental studies.

\begin{table*}[t]
\centering
\caption{Schema of the OHIE semi-synthetic dataset after preprocessing. Environment variables $E$ include treatment, state-program participation, prior emergency-department utilization, and in-person clinical measures. Proxies are clean biomedical measurements and the outcome is log-transformed ED spending.}
\label{tab:ohie_schema}
\small
\begin{tabular}{p{0.18\textwidth} p{0.25\textwidth} p{0.52\textwidth}}
\toprule
\textbf{Variable Type} & \textbf{Name(s)} & \textbf{Description} \\
\midrule

\textbf{Outcome} &
$Y_{\mathrm{obs}}$ &
Log-transformed emergency department spending ($0$--$11.13$). \\

\midrule
\textbf{Proxies} &
$Y_{\mathrm{proxy},1}$--$Y_{\mathrm{proxy},5}$ &
Biomedical measurements: systolic BP, diastolic BP, HDL, A1C, PHQ depression score. \\

\midrule
\textbf{Environment $E$: Treatment} &
$E_1$ (T) &
Lottery selection indicator (0/1). \\

\midrule
\textbf{Environment $E$: State program participation} &
\begin{minipage}{0.25\textwidth}
\raggedright
$E_2$: SNAP household (prenotify) \\
$E_3$: TANF household (prenotify) \\
\end{minipage}
&
SNAP/TANF participation prior to randomization
(large variation: 286--2952 unique values). \\

\midrule
\textbf{Environment $E$: Prior ED utilization} &
\begin{minipage}{0.25\textwidth}
\raggedright
$E_4$: num\_visit\_pre\_cens\_ed \\
$E_5$: num\_on\_pre\_cens\_ed \\
$E_6$: num\_off\_pre\_cens\_ed \\
$E_7$: num\_chron\_pre\_cens\_ed \\
$E_8$: num\_inj\_pre\_cens\_ed \\
$E_9$: num\_hiun\_pre\_cens\_ed \\
$E_{10}$: num\_loun\_pre\_cens\_ed \\
$E_{11}$: ed\_charg\_tot\_pre\_ed \\
$E_{12}$: charg\_tot\_pre\_ed \\
\end{minipage}
&
Utilization counts and total pre-period ED charges. Reflects burden of chronic,
injury-related, heart, head, and psychiatric emergency visits. Many exhibit wide 
ranges (e.g., \$0--\$180{,}054 in total charges). \\

\midrule
\textbf{Environment $E$: In-person clinical measures} &
\begin{minipage}{0.25\textwidth}
\raggedright
$E_{13}$: tot\_med\_spend\_other\_inp \\
\end{minipage}
&
Supplementary in-person medical expenditure measure (0--92{,}475 across subjects). \\

\midrule
\textbf{Bias latent} &
$A$ &
Synthetic binary latent generated from a logistic function of all environment variables. \\

\bottomrule
\end{tabular}
\end{table*}

\begin{table*}[t]
\centering
\caption{Schema of the JOBS semi-synthetic dataset after preprocessing. Environment variables consist of treatment and demographic covariates; proxies are clean pretreatment earnings; the outcome is log-transformed 1978 earnings.}
\label{tab:jobs_schema}
\small
\begin{tabular}{p{0.18\textwidth} p{0.25\textwidth} p{0.52\textwidth}}
\toprule
\textbf{Variable Type} & \textbf{Name(s)} & \textbf{Description} \\
\midrule

\textbf{Outcome} &
$Y_{\mathrm{obs}}$ &
Log-transformed 1978 earnings ($\log(1+\mathrm{re78})$), range $0$--$11.01$, with 457 unique values. \\

\midrule
\textbf{Proxies} &
$Y_{\mathrm{proxy},1}$, $Y_{\mathrm{proxy},2}$ &
Clean pretreatment earnings:
\begin{itemize}\setlength{\itemsep}{1pt}
    \item $Y_{\mathrm{proxy},1}$ (re74): range $0$--$35{,}040$, 358 unique values
    \item $Y_{\mathrm{proxy},2}$ (re75): range $0$--$25{,}142$, 356 unique values
\end{itemize}
These proxies are unaffected by the injected measurement bias. \\

\midrule
\textbf{Environment $E$: Treatment} &
$E_1$ (treat) &
Random assignment to the job-training intervention (binary 0/1). \\

\midrule
\textbf{Environment $E$: Demographics} &
\begin{minipage}{0.25\textwidth}
\raggedright
$E_2$: age (16--55; 40 unique values) \\
$E_3$: educ (0--18; 19 unique values) \\
$E_4$: black (0/1) \\
$E_5$: hispan (0/1) \\
$E_6$: married (0/1) \\
$E_7$: nodegree (0/1) \\
\end{minipage}
&
Baseline demographic variables from the NSW/JOBS dataset, representing socioeconomic and educational background. \\

\midrule
\textbf{Bias latent} &
$A$ &
Synthetic binary latent generated from a logistic function of the environment variables, used to inject controlled bias into the observed outcome. \\

\bottomrule
\end{tabular}
\end{table*}

\clearpage

\subsection{More Details on Real-world Case Study: SHELDUS}
\label{app:real}
The Spatial Hazard Events and Losses Database for the United States (SHELDUS) provides
county-level records of direct economic losses attributed to natural hazards, including
thunderstorms, floods, hurricanes, wildfires, and winter storms. Each entry aggregates
property damage, crop damage, and casualty counts reported by local agencies, making the
dataset a rich source for studying disaster impacts but also one in which measurement
error is well documented. Loss estimates often depend on heterogeneous reporting practices,
varying administrative capacity across counties, and inconsistent assessment procedures,
all of which introduce systematic biases into the observed outcomes. To support our
calibration task, we augment SHELDUS with environment variables capturing socioeconomic
characteristics (e.g., income, population, educational attainment), hazard-type indicators,
and exposure features, while also incorporating proxy variables such as remote-sensing
damage indicators or alternative loss metrics when available. This setting provides a
realistic testbed in which the true underlying losses are unobserved, reporting bias is
substantial, and proxy-guided calibration has the potential to correct distortions in
disaster loss estimation across heterogeneous counties.

\begin{table}[!htbp]
\centering
\caption{Variables used in the SHELDUS-based real-world case study after preprocessing.}
\label{tab:sheldus_variables}
\small
\begin{tabular}{p{3.4cm} p{9.0cm}}
\toprule
\textbf{Category} & \textbf{Variables and description} \\
\midrule
Identifiers &
FIPS (county identifier), Year, Month \\

Hazard indicators &
Hazard type (Flooding, Hurricane/Tropical Storm, Tornado, Wildfire), one-hot encoded as
$E_{\text{hazard},k}$ \\

Observed outcome &
$Y_{\mathrm{obs}}$: reported county-level economic loss (property damage, inflation-adjusted to 2023 USD),
log-transformed and standardized \\

Proxy variables &
$Y_{\mathrm{proxy},1}$: fraction of land transitioning from bare ground to water \\
& $Y_{\mathrm{proxy},2}$: fraction of built area transitioning to damage-related classes \\
& $Y_{\mathrm{proxy},3}$: fraction of built area transitioning to water \\
& $Y_{\mathrm{proxy},4}$: fraction of vegetation transitioning to damage-related classes \\
& $Y_{\mathrm{proxy},5}$: fraction of vegetation transitioning to water \\
& (all proxies log-scaled with factor $10^6$ and standardized) \\

Socioeconomic environment ($E$) &
Population, median age, median household income, per-capita income, median home value,
median rent, poverty count and rate, education counts (high school, bachelor’s) and percentages,
education skew \\

Housing and labor characteristics &
Housing units, household count, housing unit density, household size,
labor force size, employment and unemployment counts and rates,
labor force participation rate \\

\bottomrule
\end{tabular}
\end{table}

\clearpage

\section{Remarks on Identifiability}
\label{app:proofs}
\label{app:identifiability}

Our approach combines latent-variable modeling with causal adjustment to estimate
bias-adjusted outcomes.
This raises two distinct, but often conflated questions of identifiability:
(i) identifiability of the latent representations learned by variational autoencoders,
and (ii) identifiability of the causal estimand of interest once an appropriate
representation is available.
We emphasize that our goal is the latter; latent recovery is required only insofar
as it suffices to identify the bias estimand.

\paragraph{Latent identifiability with VAEs.}
It is well known that latent variables in deep generative models are not identifiable
without additional structure.
Recent work formalizes this limitation and provides a hierarchy of identifiability
results under increasingly strong assumptions on the prior and decoder.
In particular, under mixture priors over latents and piecewise affine decoders
(such as ReLU networks), latent variables are identifiable up to an invertible affine
transformation.
This is the weakest but most general identifiability guarantee in this hierarchy
(Table~1 in (\cite{kivva2022identifiability}), and it applies directly to standard VAE architectures.

Importantly, this form of identifiability does not imply semantic disentanglement
or recovery of a canonical coordinate system for the latent variables.
Rather, it guarantees that the learned representation is equivalent to the true
latent variable up to affine reparameterization.
Throughout this work, we explicitly adopt this weakest guarantee and do not assume
any stronger notion of latent identifiability.

\paragraph{Why affine identifiability is sufficient for our setting.}
Our method uses the learned latent $Z$ exclusively as a \emph{content or adjustment
representation}.
All downstream operations: nearest-neighbor matching in latent space, conditioning
in regression, and computation of absolute conditional average treatment effects are
invariant under affine transformations of $Z$.
Consequently, the affine identifiability result is sufficient for
our purposes: the latent representation need not be uniquely defined, only stable
up to transformations that preserve relative geometry (\cite{kivva2022identifiability}).

This distinction is crucial.
We do not require recovery of the true causal variables in a structural sense, nor
do we require disentanglement of all generative factors.
Instead, we require that the representation preserve the information necessary to
block spurious associations between the bias variable and the observed outcome.

\paragraph{A measurement-model perspective on learned representations.}
A complementary view is provided by the measurement-model framework for causal
representation learning.
From this perspective, the learned latent $\widehat Z$ is interpreted as a
measurement variable generated from the true (unobserved) content variable $Z$
via an unknown measurement function.
A representation is said to be \emph{causally valid} for a downstream estimand if it
can be used as a drop-in replacement without altering the estimand’s value
(\cite{yao2025third}).

Crucially, standard adjustment and matching estimands are invariant to invertible
reparameterizations of adjustment variables.
If $\widehat Z = h(Z)$ for a bijective (or affine) function $h$, then backdoor-style
functionals expressed in terms of $Z$ can equivalently be expressed in terms of
$\widehat Z$.
By contrast, when the treatment or outcome itself is only identified up to an
unknown reparameterization, average treatment effects are generally not identified
without additional scale-setting information.
In our setting, $Z$ is used strictly as an adjustment representation, while
$Y_{\mathrm{obs}}$ is directly observed, placing us in the invariant regime.

\paragraph{Relation to identifiability results in causal representation learning.}
Recent work on causal representation learning studies conditions under which latent
causal variables can be uniquely recovered from observations (\cite{von2023nonparametric,vargici2025score}).
These results typically rely on strong assumptions, such as access to multiple
interventional environments, independent causal mechanisms, non-Gaussian noise,
temporal ordering, or known graph constraints.
Such assumptions are appropriate when the goal is recovery of the latent causal
variables themselves.

Our objective is intentionally weaker.
Rather than identifying the full set of causal variables or the latent causal graph,
we seek identification of a specific bias-adjusted estimand.
Recovery of the true causal variables is sufficient but not necessary for this goal.
Accordingly, we do not invoke the stronger assumptions required by causal
representation learning identifiability theorems.
Instead, we leverage weaker but more broadly applicable guarantees, latent recovery
up to affine equivalence, combined with causal adjustment arguments that are invariant
to such reparameterizations.

\paragraph{Implications for bias estimation.}
Under the assumed causal structure, conditioning on the content representation $Z$
and observed covariates blocks all backdoor paths between the bias variable and the
observed outcome.
Once such a representation is available, the bias-adjusted estimand is identified
via standard adjustment.
The role of the latent-variable model is therefore not to discover the true causal
variables, but to construct a representation sufficient to satisfy the identifying
conditions of the causal estimand.

Taken together, these arguments clarify the scope of our identifiability claims.
We rely on the weakest form of latent identifiability guaranteed by modern VAE theory,
and we show that this is sufficient when combined with causal invariance
considerations to identify the bias effects of interest in real-world data.

\section{Model Architecture and Training Details}
\label{app:model}
This appendix provides additional details on the representation-learning model used in our
proxy-guided calibration framework. The model consists of two coordinated variational
autoencoders, a proxy VAE (\textsc{ZVAE}) that learns latent content factors and a bias VAE
(\textsc{AVAE}) that captures latent reporting bias affecting the observed outcome. The
architecture, objectives, and training protocol described below apply uniformly across synthetic,
semi-synthetic, and real-world datasets.

\subsection{Two-Stage Latent Variable Architecture}

\paragraph{Proxy VAE (\textsc{ZVAE}).}
The first stage models the relationship between proxies $\{Y_{\text{proxy},k}\}_{k=1}^K$ and the
latent content variable $Z$. The encoder receives the proxy vector and (optionally) environment
covariates $E$, producing mean and variance parameters for a Gaussian posterior
$q_\phi(Z \mid Y_{\text{proxy}}, E)$. The decoder reconstructs the proxies independently via a
diagonal Gaussian likelihood. The \textsc{ZVAE} thereby isolates a low-dimensional content
representation that is predictive of both $Y_{\text{true}}$ (in synthetic and semi-synthetic setups)
and the systematic components of the proxies in real data. All encoders and decoders are
two-layer MLPs with ReLU activations, batch normalization, and hidden dimension 256.

\paragraph{Bias VAE (\textsc{AVAE}).}
The second stage models the observed outcome $Y_{\text{obs}}$ as depending on both the content
latent $Z$ and an additional bias latent $A$. The \textsc{AVAE} encoder takes
$(Y_{\text{obs}}, Z, E)$ as input and infers a posterior distribution $q_\psi(A \mid Y_{\text{obs}}, Z, E)$.
The decoder reconstructs $Y_{\text{obs}}$ given $(Z, A)$ through a Gaussian likelihood with learned
variance. This stage forces $A$ to capture variation in $Y_{\text{obs}}$ not explained by $Z$ alone,
thus corresponding to systematic measurement error or reporting bias. The architecture
matches that of \textsc{ZVAE}, using two-layer MLPs with hidden dimension 256.

\subsection{Training Objective}

Each VAE is trained with a standard evidence lower bound (ELBO):
\[
\mathcal{L}_{\text{VAE}} = 
\mathbb{E}_{q(z)}[-\log p(x \mid z)]
\;+\;
\beta \, D_{\mathrm{KL}}(q(z) \,\|\, p(z)),
\]
where $\beta$ is a tunable KL weight.
For \textsc{ZVAE}, $x$ corresponds to the proxy vector; for \textsc{AVAE}, it corresponds to
$Y_{\text{obs}}$. In all experiments we use a spherical standard normal prior for $Z$ and $A$.
To ensure numerical stability across datasets with different scales, all features except $Y_{\text{obs}}$
are standardized; $Y_{\text{obs}}$ is optionally log-transformed according to dataset rules.

\subsection{Training and Evaluation Protocol}

All experiments are conducted using $k$-fold cross-validation with $k=10$. Models are
reinitialized and trained independently on each fold. During training, reconstruction
losses, latent recovery metrics, and diagnostic curves are logged for each fold. Evaluation is performed on held-out splits using the following metrics:
\begin{itemize}
    \item Permuted $R^2$ for reconstruction of $Z$ and $A$,
    \item RMSE for proxy reconstruction and $Y_{\text{obs}}$ reconstruction,
    \item Absolute error in the estimated bias parameter $\alpha$ (synthetic and semi-synthetic),
\end{itemize}
Within each fold, model selection is based on minimizing validation reconstruction loss.

Models are trained using the Adam optimizer with learning rate $10^{-3}$ and batch size 512.
Training is performed for a fixed number of epochs (50 for synthetic datasets and up to 100
for real-world datasets), without early stopping. The variational objective is optimized via
the standard ELBO with a fixed KL weight, which was sufficient to prevent latent collapse in
all settings considered due to the low-dimensional latent structure and proxy-based
supervision.





\subsection{Implementation Details}

All experiments use PyTorch models implemented with modular encoder/decoder components.
Tabular datasets follow a unified interface that partitions covariates into
$(E, Y_{\text{proxy}}, Y_{\text{obs}})$ and standardizes them according to dataset-specific rules.
The training pipeline handles fold splits, metric computation, checkpointing, ELBO logging,
and PDF report generation. Figures in the main text and appendix are produced using seaborn
and matplotlib, and performance summaries are exported as both CSV and \LaTeX{} tables.

\section{Generative AI Disclosure}
\label{app:disclosure}
Portions of this manuscript were prepared with the assistance of generative AI tools (e.g., OpenAI’s ChatGPT) for tasks such as coding, figure generation, and help with drafting. All AI-assisted outputs were carefully reviewed and validated by the authors. The conceptual framework, experimental design, and scientific conclusions are solely the responsibility of the authors.

\end{document}